\documentclass{article}

\usepackage[margin=3cm]{geometry}

\RequirePackage[colorlinks,citecolor=blue,urlcolor=blue]{hyperref}

\usepackage[numbers]{natbib}
\usepackage{amsmath}
\usepackage{amsfonts} 
\usepackage{mathtools}
\usepackage{bbm}
\usepackage{amssymb}
\usepackage{xcolor}
\usepackage[font=small]{caption}
\usepackage{subcaption}
\usepackage{tikz}
\usetikzlibrary{positioning,shapes,graphs}
\usepackage{array}
\usepackage{booktabs}
\usepackage[scr]{rsfso}
\usepackage{authblk}

\usepackage[shortlabels,inline]{enumitem}

\usepackage{arash_macros}

\newcommand\mut{\widetilde \mu}

\newcommand\wb{\bm w}
\newcommand\Nc{\mathcal N}
\newcommand\xik{\xi^{(k)}}
\newcommand\xibk{\bar{\xi}^{(k)}}
 
\newcommand{\stirling}{\genfrac{\lbrace}{\rbrace}{0pt}{}}
\usepackage{stmaryrd}
\SetSymbolFont{stmry}{bold}{U}{stmry}{b}{n}
\usepackage{bm}
\usepackage{silence} 
\newcommand\bracket[1]{\llbracket #1 \rrbracket}

\newcommand\pmax{p_{\max}}
\newcommand\walk{\mathcal{W}}
\newcommand\cluster{\mathcal C}
\newcommand\ind{\mathbbm{1}}
\newcommand{\bcdot}{\bm \cdot}
\newcommand{\bmu}{M}
\newcommand{\proj}{\mathfrak{p}}
\newcommand\Thi{T^{\rm hi}}
\newcommand\Thit{\widetilde{T}^{\rm hi}}
\newcommand\Tlo{T^{\rm lo}}
\newcommand\const{\kappa}
\newcommand\dev{\widebar{D}}
\newcommand\Imat{E}
\newcommand\Ymat{Y}

\newcommand\maxnorm[1]{\norm{#1}_{\max}}

\newcommand{\tuples}[2]{\mathcal{P}_{#1}^{#2}}

\newcommand\subg{\operatorname{SG}}

\title{Sharp Bounds for Poly-GNNs and the Effect of %
Graph
Noise}%

\author{Luciano Vinas}
\author{Arash A. Amini}
\affil{Department of Statistics, University of California, Los Angeles\\ 
          \texttt{\{lucianovinas,aaamini\}@g.ucla.edu}}

\usepackage{mdframed}

\newmdenv[
  backgroundcolor=red!20,
  linecolor=red,
  innerleftmargin=10pt,
  innerrightmargin=10pt,
  innertopmargin=10pt,
  innerbottommargin=10pt,
  skipabove=10pt,
  skipbelow=10pt,
  leftmargin=0pt,
  rightmargin=0pt,
  linewidth=2pt
]{deletion}

\begin{document}

\maketitle

\begin{abstract}
We investigate the classification performance of graph neural networks with graph-polynomial features, poly-GNNs, on the problem of semi-supervised node classification.
We analyze poly-GNNs under a general contextual stochastic block model (CSBM) by providing a sharp characterization 
of the rate of separation between classes in their output node representations. A question of interest is whether this rate depends on the depth of the network $k$, i.e., whether deeper networks can achieve a faster separation? %
We provide a negative answer to this question: for a sufficiently large graph, a depth $k > 1$ poly-GNN exhibits the same rate of separation as a depth $k=1$ counterpart. %
Our analysis %
highlights and quantifies the impact of
``graph noise'' %
in deep GNNs and shows how
noise in the graph structure can %
dominate other sources of signal in the graph,
negating any benefit further aggregation provides. Our analysis also reveals subtle differences between even and odd-layered GNNs in how the feature noise propagates.

\end{abstract}

\section{Introduction}

Graph neural networks (GNNs), like other deep learning models, %
have been shown to be best-in-class in empirical performance relative to standard kernel methods~\cite{Yang16,Kipf17}. Choices of architecture, non-linearity, and most importantly depth, are major determinants of GNN performance. However, outside of an empirically-based selection, there is little in the way of theoretical understanding as to %
why one choice of parameter may work better than another.

A parameter of particular interest is the depth of the network. A central dogma of deep learning is that deeper networks are better. They are easier to optimize and achieve better performance than their shallow counterparts~\cite{Krizhevsky12}.
However, in the case of graph machine learning, this is not always the case. The phenomenon, known as GNN \emph{oversmoothing}, is well-documented~\cite{Rusch23} and represents a departure from commonly held deep learning beliefs.

In this paper, we %
study the theoretical implications of GNN depth in a common graph learning task, namely, the semi-supervised node classification (SSNC). 
We consider community-structured graphs where both the graph structure and the node features are allowed to be noisy. In our study, we derive exact rates of misclassification and add nuance to the discussion of GNN oversmoothing. Importantly, we show that there is a fundamental misclassification rate which is available to all GNNs with polynomial features. Furthermore, this rate is sharp and invariant to network depth for sufficiently large graph inputs. %

\subsection{SSNC with GNNs}

In the task of SSNC, one is given a graph, often in the form of an adjacency matrix $A \in \{0,1\}^{n\times n}$, and is asked to make predictions using a partially observed set of labels. More formally, each node~$i$ has a feature vector $x_i \in \reals^d$ and a label $y_i \in [L] := \{1,\dots,L\}$ determining its class. We observe the graph, $A$, all the node features, collected in a matrix $X \in \reals^{n \times d}$ where the $i$th row is $x_i^T$, and a subset of the node labels $y_i, \, i \in \mathcal O \subset [n]$. The goal is to predict the unseen labels $y_i, i \in \mathcal O^c$.

The prototypical GNN design is defined layer-wise where, for $Z^{(0)} = X$, intermediate feature $Z^{(\ell + 1)}$ are expressed as
\begin{align}\label{eq:gnn:general}
    Z^{(\ell\textbf{}+1)} = \varphi( A Z^{(\ell)} W^{(\ell)}).
\end{align}		
Here, $\ell = 0,1,\dots,k-1$ denotes 
the layer index, $\varphi:\reals\to\reals$ is a non-linear function applied elementwise, and $W^{(\ell)} \in\reals^{d_{\ell}\times d_{\ell-1}}$ is the weight matrix for layer $\ell$. %
The rows of $Z^{(\ell)} \in \reals^{n \times d_\ell}$ are the (latent) node representations produced by GNN at layer $\ell$. One hopes that, by adding more layers and repeatedly aggregating over the graph, the final representation $Z^{(k)}$ will be %
more informative compared to the initial node features $Z^{(0)} = X$.%

In practice, one may replace $A$ with other graph aggregation operators $P$, such as the Laplacian matrix (in all its variants), and even change the aggregation operator per layer $P^{(\ell)}$, such as in the case of graph attention networks~\cite{Velickovic18}. Similarly, the nonlinearity can be varied layerwise $\varphi^{(\ell)}$; nonetheless we will focus on the form given in~\eqref{eq:gnn:general}.

Critically, recent empirical work~\cite{Wang22} 
have suggested that one can drop the nonlinearity $\varphi$, essentially replacing it with the identity function without noticeable change in performance on various SSNC benchmarks. %
Taking $\varphi$ to be the identity map, we can recursively unravel the layers to obtain a  simple form for $Z^{(k)} = A^{k} X W^{(0)} \cdots W^{(k-1)}$. Reparametrizing the product of weight matrices into a single weight matrix $W$, we obtain 
\[
Z^{(k)} = A^{k} X W,
\]
the basic polynomial (or rather monomial) GNN which is the object of study in this paper, which we call the \emph{poly-GNN}. Training a classifier for the poly-GNN %
amounts to first forming the graph-aggregated features $A^{k} X$ and then training a linear classifier on the observed pairs $\bigl( (A^{k} X)_{i*}, y_i \bigr), i \in \mathcal O$, where $(\cdot)_{i*}$ denotes the operator that extracts the $i$th row of a matrix.

Given this framing, our interest naturally lies in the predictive ability of graph-aggregated features $\phi^{(k)} \coloneqq A^k X \in  \reals^{ n \times d}$. 
In this case, the pivotal quantity to explain performance is the feature signal-to-noise ratio (SNR) for $\phi^{(k)}$: 
\begin{equation}\label{eq:rho:def}
    \frac1{\rho^{(k)}}
    \coloneqq \min_{i,j:\,y_i\neq y_j} \frac{\norm{\ex[\phi_i^{(k)}] - \ex[\phi_j^{(k)}]}_2}{\big(\frac{1}{n}\sum_{i'}\norm{\phi_{i'}^{(k)} - \ex[\phi_{i'}^{(k)}]}_2^2\big)^{1/2}}.
\end{equation}
where $\phi_i^{(k)}$ is the $i$th row of $\phi^{(k)}$ viewed as a column vector. Specific to misclassification, a %
feature SNR $(\rho^{(k)})^{-1}$ %
which increases with 
the sample size, $n$,
is sufficient
for the recovery of node labels $y$. Additionally, the feature SNR parameterizes the misclassification error, and is strictly more flexible than other related notions, such as linear separability.

\subsection{CSBM and Noise Decompositions}
\label{sec:csbm:noise:decomp}

A suitable theoretical model for SSNC and, by extension, for the aggregated features $\phi^{(k)}$, is the contextual stochastic block model (CSBM)~\citep{Deshpande18}. The CSBM is an extension of the stochastic block model (SBM)~\cite{Holland83} where latent labels $y = (y_i)_i$ 
determine both the distribution of the random edge $A_{ij}$ and that of the node  feature $x_i$. 

In particular, network data $(A,X)$ is said to be CSBM-generated if, for some cluster centers $\mu_1,\ldots,\mu_L\in\mathbb{R}^d$ and a connectivity matrix $B\in\mathbb{R}^{L\times L}$, the data follows
\[
x_i\,|\,y_i \sim \mu_{y_i} + \eps_i,\qquad\text{and}\qquad A_{ij} \,|\, y_i,y_j\sim \text{Bern}(B_{y_iy_j}),
\]
with zero-mean, sub-Gaussian noise vector $\varepsilon_i$. %
Originally, CSBM referred to the case where $B_{\ell\ell} = p$ and $B_{k\ell} = q$ for all $k \neq \ell$. We refer to this special case as $(p,q)$-CSBM.

The CSBM allows us to 
give a simple intuition for why graph aggregation in a GNN leads to better features for classification. Consider the matrix monomial $AX$, describing a first-order neighbor aggregation on the features $X$. For $X$ which is CSBM-generated (and more generally for any $X$ which contains additive noise), the first-order aggregation can be decomposed as
\begin{align}\label{eq:naive:decomp}
(AX)_{i *}^T = \sum_j A_{ij}\mu_{y_j} + \sum_j A_{ij}\varepsilon_j.
\end{align}
Consider an ideal CSBM where $B_{k\ell} = 0$ for $k\neq \ell$ and $B_{k k} > 0$ for all $k$. For simplicity, for this example, assume $L=2$, $\mu_1 = 1$, $\mu_2 = -1$, and $\eps_i \sim N(0,1)$. Then, if node $i$ is in cluster $1$, it will only be connected to nodes in cluster~$1$, hence the first term above is equal to $\text{deg}(i) \, \mu_1 = \text{deg}(i)$ where $\text{deg}(i)$ is the degree of node $i$. On the other hand the second term is a sum of $\text{deg}(i)$ independent zero mean variables, hence has standard deviation on the order of $\sqrt{\text{deg}(i)}$. That is, the noise grows slower than the signal \big($\sqrt{\text{deg}(i)}$ versus $\text{deg}(i)$\big), %
improving the SNR
after one round of aggregation.

The above observation is the principal idea familiar to statisticians that averaging reduces noise. It also suggests that the left and right terms in~\eqref{eq:naive:decomp} are the signal and noise, respectively. However, this description is deceiving in a general CSBM where $B_{k \ell} > 0$ for $k \neq \ell$.  In the general case, each $A_{ij}$ contains additional information about the generation process, both signal and noise, hence the first term in~\eqref{eq:naive:decomp} carries noise as well. 

\paragraph{Noise Decomposition}
Suppose that cluster centers $\{\mu_\ell\}$ were known. Then, there is an idealized aggregation operator $\mathbb{E}[A]$ that would maximally enhance the signal if it was available in place of $A$, that is, 
$\sum_j \ex [A_{ij}]\mu_{y_j}$ rather than $\sum_j A_{ij}\mu_{y_j}$ should be considered the true signal. The price we pay in the signal for not knowing $\mathbb{E}[A]$ is 
$\sum_j(A_{ij}-\ex[A_{ij}])\mu_{y_j}$ and can be referred to as the \emph{graph noise} $\Delta$. 

A mirror image of the above is obtained by examining feature $\ex[A]x_j$, assuming we know $\ex[A]$. Here the %
maximally enhanced version
is obtained by replacing $x_j$ with the ideal center $\mu_{y_j}$.
The associated cost of not knowing $\mu_{y_j}$ is $\ex[A]\varepsilon_j$ and will be referred to as \emph{aggregated feature noise} $\Delta^\varepsilon$ or just feature noise for short.

These noise terms hold more generally for a $k$-hop aggregated feature where
\begin{align*}
    \phi_i^{(k)} - \ex[\phi_i^{(k)}] &= (A^k X)_{i * }^T - \ex[(A^k X)_{i * }^T]\\
    &= \sum_j\big((A^k)_{ij} - \ex[A^k]_{ij}\big)\mu_{y_j} + \sum_j \ex[A^k]_{ij}\varepsilon_j + \sum_j\big((A^k)_{ij} - \ex[A^k]_{ij}\big)\varepsilon_j\\
    &\eqqcolon \sum_j 
    \Delta^\mu_{ij} +
    \Delta_{ij}^\varepsilon + \widetilde{\Delta}_{ij}.
\end{align*}
For both $\ex[A]$ and $\mu_{y_j}$ unknown, an additional noise interaction $\widetilde{\Delta}$ is introduced. This additional term can be absorbed as a graph noise, where $\Delta_{ij}\coloneqq 
\Delta_{ij}^\mu +
\widetilde{\Delta}_{ij}$ now acts on the random centers $x_j = \mu_{y_j} + \varepsilon_j$.

\paragraph{Walk Decomposition of Noise}
Each $(A^k)_{ij}$ admits a linear decomposition $\sum_{w \in \mathcal{W}_k} A_w$ which pass on to similar decompositions for $\Delta$ and $\Delta^\varepsilon$. Here, $w = (i_\ell, i_{\ell+1})_\ell$ is a (directed) walk of length $k$ on the complete graph on $[n]$, and $\mathcal W_k$ is the set of such walks. The notation $A_w$ is shorthand for the product of edges along the walk, that is, $A_w =\prod_{\ell=1}^{k} A_{i_\ell i_{\ell+1}}$. The variance of walk product $A_w$
scales with the number of unique edges %
in $w = (i_\ell, i_{\ell+1})_\ell$. Since edge %
direction
does not change the walk product, we can %
classify
subgraphs %
of the complete graph on
$[n]$ %
in terms of their contributions
to noise %
terms $\Delta$ and $\Delta^\eps$. In this sense, it is possible to collect walks $\mathcal{N}_\alpha \subset \mathcal{W}_k$ by %
the type 
$\alpha$ of their subgraph, %
and organize them according to their total noise contribution. 
We will show that for the problem at hand,
there is a restricted set of walks, $\mathcal{N}_*$, which overwhelming contribute to the noise. This fact both streamlines our analysis and allows for exceedingly tight approximations of the noise.

\subsection{Paper Overview}

In this paper we provide a sharp analysis of the feature SNR $\big(\rho^{(k)}\big)^{-1}$. This requires providing upper and lowerbounds for the signal %
as well as bounding and disentangling contributions from noise terms $\Delta$ and $\Delta^\eps$.

\paragraph{Technical Contributions}
In our analysis, we make the following technical contributions:
\begin{itemize}
    \item Introduce novel tools from matrix perturbation theory to generalize previous spectral structure arguments (i.e. circumventing Davis--Kahan).
    \item Define a higher-order notion of walks, \emph{walk sequences}, which naturally arise from matrix moments.
    \item Provide a complete characterizations of dominant walk structures $\mathcal{N}^*$ for both graph and feature noise under general sparsity conditions.
    \item Provide general concentrations from moments for heavy-tailed sub-Weibull~\cite{Vladimirova20} distributions. %
    \item Provide a novel analysis connecting signal structure to a necessary lowerbound %
    on
    graph noise for community structured graphs.
\end{itemize}

    In Section~\ref{sec:main_results}, we provide an overview of the main result, along with its various components (separate signal and noise bounds), and a discussion of their implications. We also discuss connections to %
    previous work at the end of this section.
    Section~\ref{sec:signal} gives a self-contained analysis of the signal component of the SNR, providing a streamlined treatment of the signal component using standard matrix concentration results and a matrix mean value theorem. %
    
    Section~\ref{sec:noise} provides a detailed walkthrough on the sources of %
    the noise component of the SNR.
    Here, we bring tools from random matrix 
    theory to bear on the analysis of GNNs.
This section
is by far the most technically involved, %
partly due to the use of general assumptions %
(see Section~\ref{sec:main_results}),
and partly due to
the natural difficulty %
of deriving high-probability results for collections of dependent quantities like $(\phi^{(k)}_i)_i$. With that said, the payoff of our analysis is clear, as we 
can provide SNR rates that are tight in the %
sample size $n$ for any fixed or 
slow growing GNN depth~$k$.

Together, Sections~\ref{sec:signal} and~\ref{sec:noise} provide the high-level proof of the main results. We have included the high-level argument in the main text since the proof provides more insights into the behavior of GNNs than what is reflected in the statement of the main results. An example is the subtle distinction between even and odd-layered GNNs in the effect of feature noise; see the discussion on dominant walks in Section~\ref{sec:feature:noise}. Another example is the characterization of the dominant walks for the graph noise in Section~\ref{sec:char_Nrt}. We conclude the paper with discussion in Section~\ref{sec:conclude}.

\paragraph{Notation.} For a vector $x = (x_i) \in \reals^d$ we write $\norm{x}_p = \bigl(\sum_i |x_i|^p \bigr)^{1/p}$ for its $\ell_p$ norm. For a matrix $X \in \reals^{n \times d}$, we write $\norm{A}_{p \to q} = \max_{\norm{x}_p \le 1} \norm{Ax}_q$ for its $\ell_p/\ell_q$ operator norm. The $\ell_2/\ell_2$ operator norm is simply written as $\opnorm{A}$. We write $\maxnorm{A} = \max_{i,j} |A_{ij}|$. For two sequences $\{a_n\}$ and $\{b_n\}$, we write $a_n \lesssim b_n$ or $b_n \gtrsim a_n$ if there is a universal constant $C > 0$ such that $a_n \le C b_n$. We write $a_n \asymp b_n$ if $a_n \lesssim b_n$ and $a_n \gtrsim b_n$. These notations will be used more generally, for any two expressions, to mean existence of the corresponding inequalities up to universal constants. We write $Z \sim \subg(\sigma)$ to denote a zero-mean sub-Gaussian random variable $Z$ with parameter $\sigma$, that is, $\ex e^{\lambda Z} \le e^{\lambda^2 \sigma^2/2}$ holds for all $\lambda \in \reals$. The complete graph on nodes $[n] := \{1,\dots,n\}$ is denoted $K_n$.

\section{Main result}\label{sec:main_results}
Let $\mathcal{C}_\ell = \{j:\, y_j=\ell\}$ denote the set of indices corresponding to the $\ell$th class. Pick some $i\in\mathcal{C}_\ell$. Then, the ideal center of $\mathcal{C}_\ell$ can be defined as
\begin{align}\label{eq:center:def}
\widetilde{\mu}_\ell^{(k)}\coloneqq \ex \phi_i^{(k)}  = \sum_j \mathbb{E}[A^k]_{ij} \mu_{y_j} = \sum_{\ell'=1}^L \sum_{j\in \mathcal{C}_{\ell'}}\mathbb{E}[A^k]_{ij}\mu_{\ell'}.    
\end{align}
By the community symmetry of CSBM, $\widetilde{\mu}_\ell^{(k)}$ is independent of which node $i\in\mathcal{C}_\ell$ was picked. As a consequence, the SNR in~\eqref{eq:rho:def} simplifies to
\begin{equation}
    \frac1{\rho^{(k)}} 
    = \frac{\min_{\ell\neq \ell'}\norm{\widetilde{\mu}_\ell^{(k)} - \widetilde{\mu}^{(k)}_{\ell'}}_2}{\big(\frac{1}{n}\sum_i\norm{\phi_i^{(k)} - \ex[\phi_i^{(k)}]}_2^2\big)^{1/2}} =  \min_{\ell \neq \ell'} \frac{
    S(\ell, \ell')}{\dev}
\end{equation}
where
\begin{align}\label{eq:S:dev:def}
   S(\ell, \ell') \coloneqq  \norm{\widetilde{\mu}_\ell^{(k)}-\widetilde{\mu}_{\ell'}^{(k)}}_2, \quad \dev \coloneqq \big(\frac{1}{n}\sum_i\norm{\phi_i^{(k)} - \ex[\phi_i^{(k)}]}_2^2\big)^{1/2}.
\end{align}
Let us now summarize the assumptions we make. We start with the sparsity regime for CSBM.
Let $p_{ij} = \ex[A_{ij}]$, and
\[
\pmax := \max_{i,j} %
p_{ij}
= \max_{\ell, \ell'} B_{\ell \ell'}, \quad \text{and} \quad \nu_n := n \pmax.
\]
Parameter $\nu_n$ captures the sparsity of the graph. %
More precisely, we assume

\begin{enumerate}[label=\text{(A\arabic*)}, ref=(A\arabic*)]
    \item \; 
    For every $\ell \in [L]$, there is $\mathcal{I}_\ell\subseteq[L]$ with 
    $|\mathcal{I}_\ell|\geq 1$
    such that
    \begin{alignat}{3}
        n B_{\ell \ell'} &\ge c_B \, \nu_n,  &\quad& \ell' \in \mathcal I_\ell, \\
        n B_{\ell \ell'} &\le C_B \nu_n^{1-\delta}, &\quad& \ell' \notin \mathcal I_\ell, \label{eq:B:lower:order:growth}
    \end{alignat}
 for some constants $c_B, C_B > 0$ and %
 $\delta \in (0, \infty]$.
 \label{as:scale}
\end{enumerate}
On the first reading, one can take $\mathcal I_\ell = [L]$ for all $\ell$, so that~\eqref{eq:B:lower:order:growth} is vacuous (we can take $\delta =\infty$ in this case for subsequent results). This case corresponds to the most common setting in the literature where one assumes the entries of $B$ all grow at the same rate $\asymp \nu_n / n$, in which case $\nu_n$ roughly corresponds to the average degree and is a measure of graph sparsity. In particular, letting $\nu_n = o(n)$ leads to asymptotically sparse graphs. The general form of assumption~\ref{as:scale} significantly relaxes the standard setting above, by only requiring at least one entry of $B$ in each row to grow at the rate similar to $\pmax = \nu_n/n$ while the other entries are allowed to decay to zero faster.

Tacking on %
the following assumption
prohibits degenerate cases where edge probabilities $p_{ij} \to 1$:
\begin{center}
\begin{enumerate*}[label=\text{(A\arabic*)}, ref=(A\arabic*),start=2]
    \label{as:scale}
    \item \; $\nu_n \le (1-c_\nu) n$ for $c_\nu \in (0, 1)$.
    \label{as:nu:upper}
\end{enumerate*}
\end{center}

Let $\pi_\ell %
= |\cluster_\ell| / n$ and let $\pi = (\pi_1,\dots,\pi_L)$ be the vector collecting the class proportions of the $L$-classes. 
 We make the following assumption on class proportions:
\begin{center}
\begin{enumerate*}[label=\text{(A\arabic*)}, ref=(A\arabic*),start=3]
    \item \;
    $L \pi_{\ell} \ge c_\pi$ and $\sqrt L \norm{\pi}_2 %
    \le C_\pi$,
    \qquad \label{as:pi}
\end{enumerate*}
\end{center}
Assumption~\ref{as:pi} requires clusters $\mathcal{C}_\ell$ to be of similar order with $\pi_\ell \asymp 1/L$.

Next, consider the matrix  $\mu$ whose $\ell$th column is the cluster mean for node features in class $\ell$: 
\begin{align}
 \mu = [\mu_1,\,\mu_2,\,\ldots,\mu_L]\in\mathbb{R}^{d\times L} 
\end{align} 
where $\mu_\ell = \ex[x_i]$ for $y_i=\ell$. We will need assumptions on the size of $\mu$ and its interaction with $B$. Consider the growth-normalized, $k$-aggregated population vectors
\begin{align}\label{eq:xibk:def}
    \bar{\xi}^{(k)}_\ell \coloneqq \mu \, \Bigr(\Pi \cdot \frac{nB}{\nu_n} \Bigl)^k e_\ell,
\end{align}
 where $ \Pi \coloneqq \diag(\pi) \in \reals^{L \times L}$ is the diagonal matrix collecting class proportions. We assume
 \begin{center}
\begin{enumerate*}[start=4, label=\text{(A\arabic*)}, ref=(A\arabic*)]
    \item \; 
    $\norm{\mu}\le C_\mu \sqrt{d}$, %
    \qquad \label{as:mu}
    \item  \quad 
    $\norm{\bar{\xi}^{(k)}_\ell - \bar{\xi}^{(k)}_{\ell'}}_{2} \ge c_\xi \sqrt{d}$,
    \label{as:sep}
\end{enumerate*}
\end{center}
and, without loss of generality, take $c_\xi \le 1$. We refer to $c_\xi$ as the separation factor of the graph. Note that $\bar{\xi}^{(k)}_\ell$ are growth-normalized, since $B_{\ell\ell'} \asymp \nu_n/n$ for $\ell,\ell'\in[L]$ by Assumption~\ref{as:scale}. As a result, $c_\xi$ is free from %
scaling with $n$.

Assumption~\ref{as:mu} characterizes the growth of $\mu\in \mathbb{R}^{d\times L}$ to be $d$-dominant in operator norm. Alternatively, it can be considered the definition of constant $C_\mu$. Since we keep track of all the constants in the result, even if $C_\mu$ depends on $d$ and $L$, one can track its effect on the final bound. On the first pass, however, it is helpful to think of $C_\mu$ as constant, which is the case if, for example, the entries of $\mu$ are of order $1$, the number of classes $L$ is kept fixed and the dimension $d$ grows.

Assumption~\ref{as:sep} is the key condition of the result. 
It is a constraint that prevents any two population means from becoming indistinguishable after $k$ smoothings.
For a simple example where the condition is violated consider a balanced $(p,q)$-CSBM with $p=q$ (i.e., an Erd\"{o}s-R\'{e}nyi graph) with means $\mu = [1,-1]$. One should not be able to improve SNR by graph aggregation in this case.

\subsection{Informal statement}
Our main result shows that the SNR in $k$-hop aggregated features has strong invariance to the depth $k$ as $n$ grows:
\begin{thm}[Informal]\label{thm:snr_bounds}
    Let $(A,X)$ be generated from an $L$-class CSBM satisfying \ref{as:scale}-\ref{as:sep} with $\nu_n \gtrsim \log n$ and sufficiently large $n$. Then, for any $k\geq 1$, %
    with high probability,
    \begin{equation}\label{eq:snr_bounds:hi}
        \sqrt{\nu_n}\,\rho^{(k)} \leq C\, c^{-1}_\xi
    \end{equation}
    for a constant $C$ independent of $n$ and $k$. Furthermore, %
    with %
    probability bounded away from zero,
    \begin{equation}\label{eq:snr_bounds:lo}
        \sqrt{\nu_n}\,\rho^{(k)} \geq c\, c_\xi
    \end{equation}
    for a constant $c > 0$ independent of $n$ and $k$. %
\end{thm}

A more precise statement of the result can be found in Section~\ref{sec:formal_result} (Theorem~\ref{thm:main_result}).
Theorem~\ref{thm:snr_bounds} states that there is a fundamental rate of separation, $\sqrt{\nu_n}$, which is the same  %
for any $k$-hop feature, given that the sample size $n$ is sufficiently large.
That is, increasing $k$ neither improves nor degrades the rate of separation.
Furthermore, any $k$-dependence in the SNR must come from the composition of the separation factor $c_\xi$. In this way, we say that $\rho^{(k)}$ is \emph{rate invariant} to the poly-GNN depth $k$.

Another consequence of Theorem~\ref{thm:snr_bounds} is that graph aggregation by GNN does indeed help, compared to classifying solely based on node features $X$, precisely at the relative rate $\sqrt{\nu_n}$, whenever $\nu_n$ grows with $n$.%

Next, to illustrate the $k$-dependence found in $c_\xi$, %
consider a simple example where
$\mu$ is full rank with $\sigma_\ell(\mu) \asymp \sqrt{d}$. More precisely, assume $c_\mu \sqrt{d} \le \sigma_\ell(\mu) \le C_\mu \sqrt{d}$. Set $\widetilde{B} \coloneqq \Pi (n B /\nu_n)$, and note that under our assumptions, $c_{\widetilde{B}} \le \sigma_\ell(\widetilde{B}) \le C_{\widetilde{B}}$ for constant $c_{\widetilde{B}}$ and $C_{\widetilde{B}}$ that potentially only depend on $L$. Then,
\begin{align*}%
    \norm{\bar{\xi}^{(k)}_\ell - \bar{\xi}^{(k)}_{\ell'}}_{2} &\geq \sigma_{\rm min}(\mu) \big(\sigma_{\rm min}(\widetilde{B})\big)^k %
    \ge c_\mu c_{\widetilde{B}}^k\sqrt{d}
\end{align*}
and we can take $c_\xi = c_\mu c_{\widetilde{B}}^k$.
Similarly for an upperbound one has
\begin{equation*}%
    \norm{\bar{\xi}^{(k)}_\ell - \bar{\xi}^{(k)}_{\ell'}}_{2}  \leq \sqrt{2}\, \norm{\mu}\, \norm{\widetilde{B}^k} \leq \sqrt{2d}\, C_\mu C_{\widetilde{B}}^k.
\end{equation*}
We see that $c_\xi \asymp c^k$ which, for pessimistic estimates $c < 1$, has a deflationary effect on the SNR. This potential deflationary effect between depth $k$ and classification accuracy is not surprising and matches the well-documented GNN oversmoothing found in the machine learning community. %
Perhaps more interesting is the following facts which can be gleaned from Theorem~\ref{thm:snr_bounds}:
\begin{enumerate}
    \item Oversmoothing is a scale effect that does not influence the SNR rate.
    \item The rate optimal choice for SNR is obtained at $k=1$.
\end{enumerate}
These insights
increase our understanding of GNNs and can inform future network architecture decisions for 
semi-supervised classification problems.

\subsection{Formal statement}
\label{sec:formal_result}
In order to state the precise version of Theorem~\ref{thm:snr_bounds}, we first state a sequence of results on the upper and lower bounds of the two components of the SNR (signal and noise). Each bound requires a set of assumptions, with some more relaxed that the others. In fact, separating the assumptions reveals that the noise upper bound holds beyond CSBM, in the so-called general inhomogeneous Erd\"{o}s-R\'{e}nyi (IER) model.
Given all the pieces, we  then restate our main result in Theorem~\ref{thm:main_result}.

In our analysis of $\rho^{(k)} = \dev/(\min_{\ell\neq \ell'} S(\ell, \ell'))$, we individually characterize the growth of all signal scales $S(\ell, \ell') =  \norm{\widetilde{\mu}_\ell^{(k)}-\widetilde{\mu}_{\ell'}^{(k)}}_2$ and the noise deviation $\dev = \big(\frac{1}{n}\sum_i\norm{\phi_i^{(k)} - \ex[\phi_i^{(k)}]}_2^2\big)^{1/2}$. The high-level overview of the results are as follows:
\begin{enumerate}
    \item The signal $S(\ell,\ell')$ grows precisely at the rate $\nu_n^k$ (Theorem~\ref{thm:signal:main}).
    \item The noise $\dev$ grows precisely at the rate $\nu_n^{k-1/2}$ (Theorems~\ref{thm:noise:upper:bound} and~\ref{thm:noise:lower:bound}).
\end{enumerate}
Combining the two, the SNR grows at the precise rate $\nu_n^{1/2}$ independent of $k$.

\paragraph{Signal bounds} To control the signal, consider the growth condition
\begin{align}\label{eq:nu:growth:signal}
    \nu_n \; \ge \; \max \Bigl\{ 
    c'_\nu \log n, \frac{32 L C_\mu^2 C_k^2}{ c_\pi c_\xi^2} \Bigr\}
\end{align}
for some constant $c'_\nu > 0$, potentially different from $c_\nu$ in~\ref{as:nu:upper}, and $C_k = k 2^k(C + \sqrt{(c/c_\nu)(k+1)})^k$ where $c, C > 0$ are some universal constants (see Lemma~\ref{lem:Ak:concentration}).
\begin{thm}[Signal bounds]\label{thm:signal:main}
    Assume~\ref{as:pi}--\ref{as:sep} and growth condition~\eqref{eq:nu:growth:signal}. Then, for $\ell \neq \ell'$,
        \[
    \frac{c_\xi}2  \sqrt{d}\, \nu_n^k 
    \;\leq\; S(\ell,\ell') \;\leq\;
    \sqrt{8d} \, C_\mu C_\pi^k \,\nu_n^k.
    \]
\end{thm}
In addition to establishing a precise rate of $\nu_n^{k}$ for the signal growth, Theorem~\ref{thm:signal:main} shows that %
both the scale of $\mu$ and the 
cluster proportions $\pi_\ell$ affect the signal growth. These considerations, in addition to the cluster connectivity, are also captured in the lowerbound by the constant $c_\xi$. 
Note that Theorem~\ref{thm:signal:main} %
implies the following bound on $c_\xi \leq \sqrt{32} C_\mu C_\pi^k$. %
Similar bounds on $c_\xi$ with tighter universal constants can be obtained directly from the definition~\ref{as:sep}. Theorem~\ref{thm:signal:main} is proved in Section~\ref{sec:signal}.

\paragraph{Noise bounds} Let $\kappa_{0,m} = 4 \max\{\frac{C_1 \sigma}{\infnorm{\mu_{m*}}}, 1\}$, where $C_1$ is the universal constant in Lemma~\ref{lem:subgauss_moms}---controlling moment growth of sub-Gaussian variables. Set $\kappa_0 = \max_m \kappa_{0,m}$ and let
\begin{equation}\label{eq:rn:def}
    r_n(\epsilon) \coloneqq   
    \max\big\{r\in 2\mathbb{N}:\, 3 \bigl(\kappa_0  rk e^k\bigr)^{r}\leq \nu_n^{1-\epsilon}\big\},
\end{equation}
for some $\epsilon \in (0,1)$.
Let us also define constants
\begin{align}\label{eq:kappa:const}
    \kappa_1 = \frac{c_B c_\nu c_\pi c_\xi^2}{48L}, \quad
    \kappa_2 =  8 \big( 32 \maxnorm{\mu}^4 + (8\,C_1 \sigma)^4\big), \quad 
    \kappa_3 = \max\{ 8\,C_1 \sigma , \maxnorm{\mu}\}.
\end{align}
Consider the following growth conditions:
 \begin{align}
 \min\Big\{
 \frac{n}{k\vee (4C_\mu /c_\xi)},\,
 C_B^{-1}\nu_n^\delta\Big\} &\geq \frac{4C_\mu L}{c_\pi c_\xi},
 \label{eq:n:growth}\\
     \min\Bigl\{(2k-1)^{-2}n
     ,\, \nu_{n}^{\epsilon}\Bigr\} &\ge \frac{\kappa_1}{2 \maxnorm{\mu}^2}.\label{eq:nu:growth}%
    \end{align}
Then we have the following control of the noise $\dev$:
\begin{thm}[Noise upper bound]\label{thm:noise:upper:bound}
    Assume %
     $\nu_n \geq ke^{2(k-1)}$ and $r_n(\epsilon) \geq 2$. %
    Then for all real $r\in [2, r_n(\epsilon)]$,
    \[
    \ex[|\dev|^r] \le \bigl(\kappa_3 \sqrt{8dr}\, \nu_n^{k-1/2}\bigr)^r.
    \]
    Moreover %
    for $u \ge 8 d e$, %
    \[
    \pr \bigr( \dev \geq \kappa_3 \nu_n^{k-1/2} \sqrt{u} \bigl) \le \exp \Bigl( - \frac12 \min\Big\{ \frac{u}{4d e}, r_n \Bigr\}\Bigr).
    \]
\end{thm}
\begin{thm}[Noise lower bound]\label{thm:noise:lower:bound}
    Assume~\ref{as:scale}--\ref{as:pi},~\ref{as:sep}, the growth conditions~\eqref{eq:n:growth}--\eqref{eq:nu:growth}, and $r_n(\epsilon) \ge 4$ as defined in~\eqref{eq:rn:def}. Then, for any $\eta \in (0,1)$,
     \[
    \mathbb{P}\Big(\, \dev \geq  \sqrt{\eta \,\kappa_1 d} \,\nu_n^{k-1/2}\, \Big) \geq (1-\eta)^2  \frac{\kappa_1^2}{\kappa_2}.
    \]
\end{thm}
Theorems~\ref{thm:noise:upper:bound} and~\ref{thm:noise:lower:bound} are proved in Section~\ref{sec:noise}.
Combining these bounds we can state our main result more precisely:
\begin{thm}[Main result]\label{thm:main_result}
    Assume~\ref{as:scale}--\ref{as:sep}, growths conditions~\eqref{eq:nu:growth:signal},~\eqref{eq:n:growth}, and~\eqref{eq:nu:growth}, and $r_n(\epsilon) \ge 4$ as defined in~\eqref{eq:rn:def}. Then, for any $\alpha \ge \sqrt{2}$, with probability at least $1-\exp(-\frac12 \min\{\alpha^2,r_n(\epsilon)\})$,  we have 
    \begin{align}\label{eq:main:upper:bound}
        \sqrt{\nu_n} \,\rho^{(k)} \;\le\; \sqrt{e} \alpha  \Bigl(\frac{\kappa_3}{c_\xi}\Bigr)
    \end{align}
    Moreover, for any $\eta \in (0,1)$, with probability at least $(1-\eta)^2 \kappa_1^2 / \kappa_2$, we have
    \begin{align}
    \sqrt{\nu_n} \,\rho^{(k)}  \ge \sqrt{\frac{\eta}{8}} \cdot \frac{\sqrt{\kappa_1}}{C_\mu C_\pi^k}.
    \end{align}
\end{thm}

\begin{proof}[Proof of Theorem~\ref{thm:main_result}]
    Note that condition $\nu_n \ge k e^{2(k-1)}$ of Theorem~\ref{thm:noise:upper:bound} is automatically implied by $r_n(\epsilon) \ge 4$.
    Take $u = \alpha^2 4 de$ for $\alpha^2 \ge 2$ in Theorem~\ref{thm:noise:upper:bound}. Then, with probability at least $1-\exp(-\frac12 \min\{\alpha^2,r_n\})$, we have $\dev \le \kappa_3 \nu_n^{k-1/2} \sqrt{4de}\, \alpha$. Combined with the lower bound in Theorem~\ref{thm:signal:main}, we obtain with the same probability
    \[
    \rho^{(k)} \le \frac{\kappa_3 \nu_n^{k-1/2} \sqrt{4de}\, \alpha}{c_\xi \sqrt{d} \, \nu_n^k /2} = (\kappa_3 / c_\xi) \sqrt{e} \alpha \,\nu_n^{-1/2}
    \]
    which is the claimed upper bound. For the lower bound, it is enough to combine Theorem~\ref{thm:noise:lower:bound} with the lower bound in Theorem~\ref{thm:signal:main}.
\end{proof}

A couple of comments are in order: For the upper bound~\eqref{eq:main:upper:bound} to truly hold with high probability, we must have $r_n(\epsilon) \to \infty$ as $n\to \infty$. This is the case when $\nu_n \to \infty$. In fact, one can show that, roughly $r_n(\epsilon) \gtrsim \log \nu_n / (\log \log \nu_n)$; see Lemma~\ref{lem:growth:of:rn}. %
The noise upper bound (Theorem~\ref{thm:noise:upper:bound}) 
holds beyond CSBM, in a general IER model with $n p_{ij} \leq \nu_n$. %
On the other hand, both the signal and noise lower bounds, rely on the CSBM structure, as is evidenced by their dependence on parameter $c_\xi$ (via $\kappa_1$ in the case of the noise lower bound). One needs some form of structure for any lower bound to hold; this is clear in the case of the signal, but more subtle in the case of noise. The signal upper bound also relies on the CSBM structure.

As mentioned earlier, the binary nature of adjacency matrix $A$ allows noise $\dev$ to be described in terms of walks %
on the complete graph %
$K_n$.
As will be shown, walks which are tree-like, specifically star-like and path-like, have the largest contribution 
to the noise.
For the
aggregated feature noise ($\Delta^\eps$ introduced in Section~\ref{sec:csbm:noise:decomp}), the sparsity level $\nu_n$ influences the dominant walk type and the rate of growth, with a subtle distinction between the even and odd-layered GNNs; see Lemma~\ref{lem:Nii} and the disscusion at the end of Section~\ref{sec:feature:noise}.

We suspect something similar may be true in the case of graph noise in general.
As we show in Section~\ref{sec:char_Nrt},
under structure guarantees like that of~\ref{as:sep}, the 
dominant walk type for graph noise 
can be completely characterized.
In the absence of such guarantees, the resulting dominant walk may change, and as a result, change $\nu_n^{k-1/2}$ noise growth rate.

\subsection{Previous Work}
\label{sec:previous:work}
The work of \citet{Baranwal21} is the first %
to our knowledge to explore the separation improvement %
in first-order aggregated features (i.e., $k=1$)
for
CSBM data. %
Their %
results were %
obtained
for a $(p,q)$-CSBM with a $\nu_n \gtrsim \log^2 n$ sparsity assumption. For the aggregation, a degree-normalized %
adjacency matrix 
with self-loops
was used. In their paper, a $\sqrt{\nu_n}$ separation rate for the first-order aggregated features was recovered. This rate matches the fundamental separation rate shown in our main result. %
Our setting is more general, as it considers an $L$-class CSBM, relaxes the sparsity assumption to
$\nu_n \gtrsim\log n$ 
and considers %
$k$-aggregated features for all $k \ge 1$. It is worth noting that the case $k \ge 2$ is technically much more challenging than $k=1$ due to the dependence introduced by multi-hop aggregation.
Furthermore, by focusing on the fundamental information content %
of $\phi^{(k)}$ and not necessarily just its linear separability, we are able to streamline the signal analysis 
by using simpler tools from 
matrix analysis~\cite{Bhatia97,Bandeira16} %

A similar work by \citet{Wu23} explores the oversmoothing effect in features $\phi^{(k)} = A^k X$, assuming $X$ are normally distributed. A key claim in 
their work is that $\phi^{(k)}$ are exactly normal with distribution $\phi^{(k)}\sim {\rm Gauss}(\mathbb{E}[\phi^{(k)}], {\rm Var}(\phi^{(k)}))$. The authors claim this result follows from the linearity of the matrix $A^k$, however, this cannot be the case since even $\sum_{j}A_{ij}x_j$ is, by definition, a (scaled) mixture of Gaussians. Nevertheless, under the simplification that $\phi^{(k)} \approx \mathbb{E}[A]^k X$, the authors show that misclassification of GNNs can be described in terms of a Z-score of a standard normal. As we shall see, the approximation $\ex[A^k] \approx \ex[A]^k$ is not a bad one, especially when considering the overall size of $\norm{\ex[A^k]}_2$. However, the fact that $A^k$ is a matrix of dependent quantities complicates any high probability results for $\phi^{(k)}$.

Another work by~\citet{Wei22} derives 1-hop MAP estimators, that is estimators which are locally optimal for a given neighborhood, for the case of a $(p,q)$-CSBM with normally distributed node covariates. %
The resulting estimator bears resemblance %
to %
a ReLU GNN utilizing a first-order aggregation scheme. In their main result, sparsity and mean separation are assumed to be $\nu_n \gtrsim \log^2 n$ and $\norm{\mu_1 - \mu_2} \ll \log n /\sqrt{d}$ respectively. Additionally, this work has been recently extended by~\citet{Baranwal23} to cover %
$\ell$-hop locally optimal MAP estimators %
for CSBMs satisfying sparsity %
$\nu_n \lesssim 1$. The $\mathcal{O}(1)$ sparsity constraint plays an important role in this analysis, as the networks generated from the CSBM become locally tree-like with high probability. This in turn makes the analysis more tractable for the fixed hop case. %

Our work in this paper was partly inspired by the empirical findings we report in~\cite{Vinas23} %
where a simple single-layer GNN showed similar performance to more complicated state-of-the-art architectures on SSNC benchmarks. Within this same class of simple GNNs, the largest performance changes were observed when depth increased from $k=0$ to $k=1$. %

On the topic of graph learning outside of GNNs, there is a locus of works which revolve around enforcing a Laplacian regularization to the traditional supervised learning context. This line of work traces its roots back to the manifold learning approach proposed by~\citet{Belkin06}. Recent works consider modifying the data fidelity term~\cite{Li19} or providing minimax rates for classes of non-parametric estimators with Laplacian regularization~\cite{Green21,Green23}. In the context of multi-graph regression, there are related works~\cite{Zhou22} which consider regressing node-features with respect to multiple graph Laplacians. %

\section{Signal Analysis}\label{sec:signal}
In this section, we provide the analysis leading to the proof of Theorem~\ref{thm:signal:main} controlling the signal component of the SNR. The analysis is broken into several lemmas; the proofs are given in the text when short, otherwise deferred to the appendices.

We first introduce some notation. Let $Z \in\{0,1\}^{n\times L}$ be the cluster membership matrix for $y$, that is, $Z_{ij} = 1\{y_i = j\}$, and  consider 
\begin{align}
    P := Z B Z^T.
\end{align} 
We note that $\ex[A] = P - \diag(P)$ where $\diag(P)$ denotes the diagonal matrix with the same diagonal as $P$. We have subtracted $\diag(P)$, since we assume $A_{ii} = 0$ (no self-loops). Rewriting~\eqref{eq:center:def}, %
\begin{align}\label{eq:center:matrix:form:1}
    \mut^{(k)}_\ell = \sum_{\ell'} \sum_{j} \ex[A^k]_{ij} Z_{j \ell'} \mu_{ \bcdot \ell'} = \bigl(\mu Z^T \ex[A^k]  \bigr)_{ * i}
\end{align}
for any $i \in \cluster_\ell$. Here, we have used the symmetry of $A$ and that $\mu$ is a matrix with columns $ \mu_\ell = \mu_{* \ell}$.
This implies that $\mut^{(k)}_\ell$ is also the average, over $i \in \cluster_\ell$, of the RHS of~\eqref{eq:center:matrix:form:1}, that is,
\begin{align}\label{eq:center:matrix:form:2}
    \mut^{(k)}_\ell  = \mu Z^T \ex[A^k] \frac{\ind_{\cluster_\ell}}{n_\ell}
\end{align}
where $\ind_{\cluster_\ell} \in \{0,1\}^n$ denotes the indicator vector of cluster $\ell$, that is, $(\ind_{\cluster_\ell})_i = 1\{i \in \cluster_\ell\}$.

\subsection[Signal Proxy Growth]{Signal Proxy Growth}
In showing $\norm{\widetilde{\mu}_{\ell}^{(k)} - \widetilde{\mu}_{\ell'}^{(k)}}_2 \asymp \nu_n^k$, we first construct a  proxy $\widetilde{S}(\ell,\ell')$ where $\widetilde{S}(\ell,\ell') \asymp \nu_n^k$. Motivated by the approximation $\ex[A]^k \approx P^k$ in~\eqref{eq:center:matrix:form:2}, let us write 
\begin{align}\label{eq:xik:def}
\xik_\ell := \mu Z^T (ZBZ^T)^k \frac{\mathbbm{1}_{\mathcal{C}_\ell}}{n_\ell} 
= \bmu P^k \bar \ind_{\cluster_\ell}.
\end{align}
where we have introduced
\begin{align}\label{eq:mut:averaged}
 \bmu \coloneqq \mu Z^T \in \reals^{d \times n} \quad \text{and} \quad   \bar \ind_{\cluster_\ell} := \ind_{\cluster_\ell} / n_\ell.
\end{align} 
Next we need the following identity which can be proved by induction on $k$ (the proof is omitted):
\begin{lem}\label{lem:ZTZ:identity}
    $Z^T (Z B Z^T)^{k-1} Z B = (Z^T Z B)^k$ for any $k \ge 1$.
\end{lem}
Using this lemma, each $\xik_\ell$ can be re-expressed into a %
simpler form
\begin{align*}
    \xik_\ell &= \mu Z^T (ZBZ^T)^k Z e_\ell / n_\ell\\
    & = \mu Z^T (ZBZ^T)^{k-1} Z B \cdot Z^T Z e_\ell / n_\ell
    = \mu (Z^T Z B)^k e_\ell
    = \mu (\Pi \cdot nB)^k e_\ell
\end{align*}
where the third equality follows form $Z^T Z e_\ell / n_\ell = e_\ell$ and Lemma~\ref{lem:ZTZ:identity}, and the final equality from $\Pi = Z^T Z / n$. This allows for the following bracket bound for the growth of $\widetilde{S}(\ell,\ell')$:
\begin{lem}
    Under assumptions~%
    \ref{as:pi}--\ref{as:sep}, $c_\xi \sqrt{d}\,\nu_n^k \;\leq\;\widetilde{S}(\ell,\ell')\;\leq\; \sqrt{2d} \, C_\mu C_\pi^k \,\nu_n^k$.
\end{lem}
\begin{proof}
  By the definition of $\nu_n$ and assumption~\ref{as:pi},
  we have
\begin{align*}
\opnorm{\Pi \cdot nB} 
&\le 
 \norm{I_L}_{2 \to 1} \cdot \norm{\Pi \cdot n B}_{1 \to 2}  \\ 
&\le  \sqrt{L} \cdot \max_{\ell'} \norm{(\Pi \cdot nB)_{\bcdot \ell'}}_2 \le \sqrt{L} \cdot  \nu_n \norm{\pi}_2 \le C_\pi \nu_n.
\end{align*}
Using assumptions~\ref{as:pi} and~\ref{as:mu},
we obtain
\begin{align*}
   \widetilde{S}(\ell,\ell') \;\leq\; \norm{\mu}\, \norm{\Pi \cdot n B}^k\, \norm{e_\ell - e_{\ell'}}_2 \;\leq\; 
   \sqrt{2d} \,
   C_\mu (C_\pi \nu_n)^k.
\end{align*}
For a lowerbound, 
recalling definition~\eqref{eq:xibk:def}, we note  the identity
\begin{align}\label{eq:xik:xibk:ident}
\xik_\ell = \nu_n^k \,\xibk_\ell,
\end{align}
which by assumption~\ref{as:sep} gives $\widetilde{S}(n,k) = \nu_n^{k}\,  \norm{\bar{\xi}_\ell^{(k)} - \bar{\xi}_{\ell'}^{(k)}}_2 
   \ge c_\xi \sqrt{d}\, \nu_n^k$.
Altogether then, we have, 
\[
c_\xi \sqrt{d}\, \nu_n^k \;\leq\;
\widetilde{S}(\ell,\ell') 
\;\leq\;
 \sqrt{2d} \, C_\mu C_\pi^k \,\nu_n^k.
\]  
which is the desired result.
\end{proof}

\subsection{Signal Proxy as Leading Order Approximation}

It remains to show that $\widetilde{S}(\ell,\ell')$ is indeed close to the signal deviation $\norm{\widetilde{\mu}_{\ell}^{(k)} - \widetilde{\mu}_{\ell'}^{(k)}}_2$. We frequently use the following estimates:
\begin{lem}\label{lem:M:ic:bound}
    Under~\ref{as:pi} and~\ref{as:mu}, $\opnorm{\bmu} \le C_\mu \sqrt {nd}$ and $\norm{\bar \ind_{\cluster_\ell}} \le (L/c_\pi)^{1/2} n^{-1/2}$.
\end{lem}
\begin{proof}
    We have $\opnorm{Z^T} = \sqrt{\opnorm{Z^T Z}} = \max_{\ell'} \sqrt{n_{\ell'}} \le \sqrt {n}$, and the first claim follows from $\opnorm{M} \le \opnorm{\mu} \opnorm{Z^T}$ and~\ref{as:mu}. Moreover, $\norm{\bar \ind_{\cluster_\ell}}_2 = n_\ell^{-1/2} = (\pi_\ell n)^{-1/2} \le (L/c_\pi)^{1/2} n^{-1/2}$ by~\ref{as:pi}.
\end{proof}
Using a Banach-valued variant to the mean-value theorem~\cite{Bhatia97}, one has: %
\begin{lem}\label{lem:Ak:Pk:bound}
    $\norm{\ex[A]^k - P^k}\leq k\nu_n^{k}/n$.
\end{lem}
\begin{proof}
    Recall that $\ex[A] = P - \diag(P)$ where $\opnorm{\ex[A]}$ and $\opnorm{P}$ are upper-bounded by $n \pmax = \nu_n$ and $\opnorm{\diag(P)} \le \pmax = \nu_n / n$. 
    Then, the second statement of Lemma~\ref{lem:poly_dev} gives the desired bound.
\end{proof}
Using the same Banach-valued mean-value theorem and and sharp concentrations on $\norm{A- \ex[A]}$~\cite{Bandeira16}, 
we get the following concentration inequality for $A^k$ which is proved in Appendix~\ref{sec:proofs:for:signal}:
\begin{lem}\label{lem:Ak:concentration}
    Suppose that $\nu_n \ge  c'_\nu \log n \ge 1$ for some constant $c'_\nu > 0$.
    Then, for any integer $k \ge 1$, the spectrum of $A^k$ concentrates as 
    \[
    \mathbb{E}\norm{A^k - \ex[A]^k} \le C_k\, \nu_n^{k-1/2}.
    \]
    where $C_k = k 2^k(C + \sqrt{(c/c'_\nu)(k+1)})^k$ for some universal constants $C > 1$ and $c > 0$.
\end{lem}

Next fix $\ell$ and $\ell'$, and let $w:= \bar\ind_{\cluster_\ell} - \bar \ind_{\cluster_{\ell'}}$
Using~\eqref{eq:mut:averaged} and~\eqref{eq:xik:def},
\begin{align*}
    \widetilde{\mu}_{\ell}^{(k)} - \widetilde{\mu}_{\ell'}^{(k)} =
    M \ex[A^k]  w
    \quad
    \text{and}
    \quad
    \xik_\ell  - \xik_{\ell'} =  M P^k  w.
\end{align*}
For $\ell \neq \ell'$, we have $\norm{w}_2 = \sqrt{\norm{\bar \ind_{\cluster_\ell}}^2 + \norm{\bar \ind_{\cluster_{\ell'}}}^2} \le (2L/c_\pi)^{1/2} n^{-1/2}$ by Lemma~\ref{lem:M:ic:bound}. Moreover,
\begin{align*}
    \opnorm{\ex[A^k] - P^k} 
    &\le \opnorm{\ex[A^k] - \ex[A]^k} + \opnorm{\ex[A]^k - P^k} \\
    &\le C_k\, \nu_n^{k-1/2} + (k \nu_n^k / n) \\
    &\le 2 C_k \, \nu_n^{k-1/2}
\end{align*}
where the second line uses Lemmas~\ref{lem:Ak:Pk:bound} and~\ref{lem:Ak:concentration} and the last line uses $C_k \nu_n^{-1/2} \ge k/n$ which is satisfied since $C_k \ge k$ and $\nu_n \le n$. We are now ready to prove Theorem~\ref{thm:signal:main}.

\begin{proof}[Proof of Theorem~\ref{thm:signal:main}]
    From our earlier results 
\begin{align*}
     \big|\norm{\widetilde{\mu}_{\ell}^{(k)} - \widetilde{\mu}_{\ell'}^{(k)}}_2 -\widetilde{S}(\ell,\ell')\big| 
     &\le \norm{ M (\ex[A^k] - P^k)  w}_2 \\
     &\le \opnorm{M} \cdot \opnorm{\ex[A^k] - P^k} \cdot \norm{w}_2  \\
     &\le C_\mu \sqrt{nd} \cdot (2 C_k \, \nu_n^{k-1/2}) \cdot  (2L/c_\pi)^{1/2} n^{-1/2}  \\
     &\le  \sqrt{8dL/c_\pi}\, C_\mu C_k\, \nu_n^{k-1/2}
\end{align*}
using Lemmas~\ref{lem:M:ic:bound} and~\ref{lem:Ak:Pk:bound} in the third line. Under the growth condition condition~\eqref{eq:nu:growth:signal} we obtain $1 \ge c_\xi / 2 \ge \sqrt{8L/c_\pi} C_\mu C_k \nu_n^{-1/2}$ and
\[
\frac{c_\xi}2  \sqrt{d}\, \nu_n^k 
\;\leq\; \norm{\widetilde{\mu}_{\ell}^{(k)} - \widetilde{\mu}_{\ell'}^{(k)}}_2 \;\leq\;
\sqrt{8d} \, C_\mu C_\pi^k \,\nu_n^k
\]
which is the desired result.
\end{proof}

\section{Noise Analysis}
\label{sec:noise}
In this section, we develop %
probability bounds for the noise deviation
\begin{align*}
    \dev \coloneqq \Big(\frac{1}{n}\sum_i\norm{\phi_i^{(k)} - \ex[\phi_i^{(k)}]}_2^2\Big)^{1/2} = 
    \Big(\frac{1}{n}\sum_{i,m} D_{im}^2\Big)^{1/2}
\end{align*}
where $ D_{im}:=
     \phi_{im}^{(k)} - \ex\big[\phi_{im}^{(k)}\bigl]$, leading to the proofs of Theorems~\ref{thm:noise:upper:bound} and~\ref{thm:noise:lower:bound}.
These probability bounds will be obtained through a high-moment Markov %
bound,
by analyzing the leading term of the moments of $\ex(\dev)^{r}$ for $r \in 2 \nats$. For such $r$, the $r$th moment of the noise can be upperbound as 
\begin{equation}
    \label{eq:dev:upper:jensen}
    \ex(\dev)^r \leq \frac{d^{r/2-1}}{n}\sum_{i,m}\ex D_{im}^r,
\end{equation}
where the right-hand side follows from
Jensen inequality with expectation operator $\frac{1}{nd} \sum_{i,m}$. We further decompose the inner terms as
\begin{align*}
    D_{im} %
     &= \sum_j \big((A^k)_{ij}-\ex[A^k]_{ij}\big)x_{jm} + \sum_j \ex[A^k]_{ij}\eps_{jm} \\
 &=: \Delta_{im}^{} + \Delta^\eps_{im}
\end{align*}
We will control the moments $\ex D_{im}^r$.
For $r=2$, we have
$ \ex D_{im}^2  = \ex \Delta_{im}^2 + \ex(\Delta_{im}^\varepsilon)^2$ and, 
more generally for $r \in 2\mathbb{N}$, by the convexity of $x\mapsto x^r$,
\begin{align}\label{eq:dev:delta:decomp}
D_{im}^r
\leq 
2^{r-1}\bigl(\Delta_{im}^r + (\Delta_{im}^\varepsilon)^r\bigr).  
\end{align}
Let us first control $\ex(\Delta_{im}^\eps)^r$.

\subsection{Controlling Feature Noise}
\label{sec:feature:noise}
 Recall that $\eps_{im}$ are independent zero-mean sub-Gaussian random variables with parameter $\le \sigma$; that is, $\eps_{im} \sim \subg(\sigma)$. It follows that
 $
 \Delta^\eps_{im} \sim \subg\bigl(\big(\sigma^2 \sum_j \ex[A^k]_{ij}^2\big)^{1/2}\bigr).
 $
 We can control $\ex[A^k]_{ij}$ via a \emph{walk analysis} which will be the common theme in Section~\ref{sec:noise}. A more sophisticated version of such analysis appears in Section~\ref{sec:product:walk:analysis} where we control the graph noise.

 Let us set up some notation and terminology. A $k$-walk on $[n]$ is a walk of length $k$ in the complete graph with nodes $[n]$. We represent a $k$-walk, $w$, as an ordered tuple of \emph{directed} edges 
\begin{align}\label{eq:walk}
w = ((i_1,i_2),(i_2,i_3),\ldots,(i_k,i_{k+1}))
\end{align}
We also denote the above walk as $i_1 \to i_2 \to \cdots \to i_{k} \to i_{k+1}$.
For such a walk, we write $G(w)$ for the graph obtained by considering the nodes in $w$ and all the \emph{undirected edges} present in $w$. We often denote the number of edges in $G(w)$ by $t$, which is the number of unique undirected edges in~$w$. Figure~\ref{fig:walks:and:graphs} shows examples of walks and their corresponding graphs.
The reason for considering the ``undirected'' graph of a walk  is the symmetry of $A$. The undirected graph captures the truly independent entries of $A$ that appear in the walk.

\begin{figure}[t]
\centering
\begin{tabular}{m{.5cm} m{7cm}  m{4cm}  m{.5cm}  m{.5cm} }
\toprule
& \textbf{Walk \(w\)} & \textbf{Graph \(G(w)\)} & \textbf{\(k\)} & \textbf{\(t\)} \\[1ex]
\midrule
$w_1$ & 5 $\rightarrow$ 1 $\rightarrow$ 2 $\rightarrow$ 1 $\rightarrow$ 3 $\rightarrow$ 4 $\rightarrow$ 3 $\rightarrow$ 1 $\rightarrow$ 5 & 
\begin{tikzpicture}[node distance=1cm, auto]
  \node (n5) {5};
  \node (n1) [right of=n5] {1};
  \node (n2) [right of=n1] {2};
  \node (n3) [below of=n1] {3};
  \node (n4) [right of=n3] {4};

  \draw (n5) -- (n1);
  \draw (n1) -- (n2);
  \draw (n1) -- (n3);
  \draw (n3) -- (n4);
\end{tikzpicture} & 
8 & 4 \\[1ex]
\midrule
$w_2$ & 5 $\rightarrow$ 1 $\rightarrow$ 3 $\rightarrow$ 1 $\rightarrow$ 5 $\rightarrow$ 3 $\rightarrow$ 5 & 
\begin{tikzpicture}[node distance=1cm, auto]
  \node (n5) {5};
  \node (n1) [right of=n5] {1};
  \node (n3) [below of=n1] {3};

  \draw (n5) -- (n1);
  \draw (n1) -- (n3);
  \draw (n5) -- (n3);
\end{tikzpicture} & 
6 & 3 \\
\bottomrule
\end{tabular}
\caption{Walks and their corresponding graphs for $n=5$.}
\label{fig:walks:and:graphs}
\end{figure}

Let $\mathcal{N}_{t}(i,j)$ be the set of %
$k$-walks going from %
$i$ to $j$
with $t$ unique undirected edges. 
Representing $w \in \mathcal{N}_{t}(i,j)$ as in~\eqref{eq:walk} with $i_1 = i$ and $i_{k+1} = j$, we have 
\begin{equation}\label{eq:ij_walks}
    \mathbb{E}[A^k]_{ij} = \sum_{t=1}^k \sum_{w\in\mathcal{N}_{t}(i,j)} \mathbb{E}\Big[\prod_{\ell=1}^{k} 
    A_{i_\ell,i_{\ell+1}}
    \Big] 
\end{equation}
Necessary to our argument is the following counting lemmas on the number of $k$-walks: %
\begin{lem}\label{lem:N_tij_walks}
    $|\mathcal{N}_{t}(i,j)| \leq \binom{n-2}{t-1}t^{k-1}$ for distinct $i, j \in [n]$.
\end{lem}
The case $i = j$ is more subtle. We partition $\mathcal{N}_{t}(i,i)$ into walks $w$ whose undirected 
 graph $G(w)$ has loops,  $\bm{\mathring}\Nc_t(i,i)$, and walks for which $G(w)$ has no loops, $\bm{\breve}\Nc_t(i,i)$. As example, consider $w_1$ and $w_2$ given in Figure~\ref{fig:walks:and:graphs} and note that $w_2 \in \bm{\mathring}\Nc_t(i,i)$ while $w_1 \in \bm{\breve}\Nc_t(i,i)$.
\begin{lem}\label{lem:Nii}
We have $|\bm{\mathring}\Nc_t(i,i)| \leq \binom{n-1}{t-1} t^{k-1}$ and
\begin{align}\label{eq:catalan:stirling}
|\bm{\breve}{\mathcal{N}}_t(i,i)| \leq C_t \binom{n-1}{t}\cdot t! \,\stirling{k/2}{t}\quad k \in 2\nats,\; t \le k/2   
\end{align}
where $C_t = \frac1{t+1}\binom{2t}{t}$ is the Catalan number and $\stirling{m}{t}$ is the Stirling number of the second kind. Bound~\eqref{eq:catalan:stirling} holds with equality when $k = 2t$. A further upper bound is
\begin{align}\label{eq:loopless:crude:bound}
    |\bm{\breve}{\mathcal{N}}_t(i,i)| \le (2e^2 n)^t \,t^{k/2-t-1}.
\end{align}

\end{lem}
Note that these bounds imply that, for a given $t\leq k/2$, the walks in $|\bm{\mathring}\Nc_t(i,i)|$ have the fastest growth in $n$, of order $O(n^t)$, %
compared to walks in the other two categories whose growth is $O(n^{t-1})$. These lemmas allows us to bound adjacency moments $\ex[A^k]$ elementwise:
\begin{lem}\label{lem:Akij_growth}
   Assume %
   $\nu_n \ge k e^{2(k-1)}$,
   then 
   \[\mathbb{E}[A^k]_{ij} \;\leq\; 2\, p_{\rm max}\nu_n^{k-1} + 2 \nu_n^{k/2}\ind\{i = j, \,k \in 2 \nats\}.
   \] %

\end{lem}
\begin{proof}
 First, assume $i \neq j$.
 For $w =((i_\ell,j_\ell))$ with $t$ unique edges, we have 
$
\ex[\prod_{\ell=1}^k A_{i_\ell j_\ell}]\leq p_{\rm max}^{t} = (\nu_n /n)^t
$ which gives
\[
    \ex[A^k]_{ij} 
    \;\le \;
     \sum_{t=1}^{k}   |\Nc_{t}(i,j)| \pmax^{t} \;\le \pmax \sum_{t=1}^{k} \binom{n}{t-1} t^{k-1} \pmax^{t-1}
\]
Using $\binom{n}{t-1} \leq (en/(t-1))^{t-1}$, and $(t/(t-1))^{t-1} \le e$ for $t > 1$, we have $\binom{n}{t-1} \le e \cdot (en /t)^{t-1} $. Plugging in and noting $n \pmax = \nu_n$, we obtain
$
\ex[A^k]_{ij}
\le e \pmax \sum_{t=1}^{k} (e \nu_n / t)^{t-1} t^{k-1}.
$
Further dividing both sides by $(e\nu_n)^{k-1}$ we have
\[
\frac{\ex[A^k]_{ij}}{(e\nu_n)^{k-1}} 
\le e \pmax \sum_{t=1}^{k} (t/(e \nu_n))^{k-t}
\]
Let $\rho = k / (e\nu_n)$. By assumption $k e^{k-1} \le \nu_n$ so that $\rho \le e^{-k} < 1/2$. Then, we have 
\[
\frac{\ex[A^k]_{ij}}{\pmax \nu_n^{k-1}} 
\le e^k \sum_{t=1}^{k} \rho^{k-t} \le e^k \sum_{u=0}^\infty \rho^u \le 2 e^k \rho \le 2.
\]
which is the desired result.

\medskip

Next for $\ex[A^k]_{ii}$, we have the following bounds
\begin{align*}
\ex[A^k]_{ii} \;\leq\; \sum_{t=1}^k |\mathcal{N}_{t}(i,i)|p_{\rm max}^t  \;\leq\; \sum_{t=1}^k |\bm\mathring{\mathcal{N}}_{t}(i,i)|p_{\rm max}^t +\sum_{t=1}^{k/2} |\bm\breve{\mathcal{N}}_{t}(i,i)|p_{\rm max}^t.
\end{align*}
The first sum %
bounds exactly as above. The second %
sum is zero unless
$k\in2\mathbb{N}$, which we assume for the rest of the argument. 
Let $c = 2e^2$ and note that by~\eqref{eq:loopless:crude:bound} of Lemma~\ref{lem:Nii},
\begin{align*}
    \sum_{t = 1}^{k/2} |\bm\breve{\mathcal{N}}_t(i,i)| p_{\rm max}^t \leq \sum_{t=1}^{k/2} (c \nu_n)^t\, t^{k/2-t}
    &=   (c\nu_n)^{k/2}\sum_{t=1}^{k/2} (t/(c\nu_n))^{k/2-t} \\
    &\le (c\nu_n)^{k/2}  \sum_{u=0}^{\infty} \rho^u \le (c \nu_n)^{k/2} \cdot 2\rho \le 2 \nu_n^{k/2}
\end{align*}
where $\rho = (k/2)/(c\nu_n)$,  $\sum_{u=0}^\infty \rho^u \le 2\rho$ since $\rho < 1/2$, 
and  
\[c^{k/2} \rho = (\sqrt 2 e)^k k / (4e^2 \nu_n) \le e^{2k} k/(e^2 \nu_n) \le 1\] by assumption.
The proof is complete.
\end{proof}

This style of walk argument will appear again and in more detail as we consider the network noise components $\Delta_{im}$. For the feature noise, we now need to translate the moment bound in Lemma~\ref{lem:Akij_growth} to a concentration bound. The following is well-known~\citep[Proposition 2.5]{hdp}:
\begin{lem}\label{lem:subgauss_moms}
	If $Z$ is sub-Gaussian with parameter $\sigma$, then,
 $\ex|Z|^r \le (C_1 \sigma r^{1/2})^r$ where $C_1$ is a numerical constant. %
\end{lem}
The reverse is also true in the sense that a moment growth of the form above implies $Z$ is sub-Gaussian. This also follows from \citep[Proposition 2.5]{hdp}. Alternatively, it follows from the more general Lemma~\ref{lem:moment:concent1} (Appendix~\ref{app:moment:concent}) with $\eta = 1/2$.
The bound in Lemma~\ref{lem:Akij_growth} gives (using $\sqrt{a +b } \le \sqrt{a} + \sqrt{b}$):
\begin{align}
\big(\sum_j \ex[A^k]_{ij}^2\big)^{1/2} &\le \ex[A^k]_{ii}  + \sqrt{n} \max_{j\neq i} \ex[A^k]_{ij} \notag \\
&\le 
2 \nu_n^{k/2} \ind\{k \in 2\nats\} + 2 \pmax^{1/2} \nu_n^{k-1/2} \notag \\  
&\le 2 \Bigr(\nu_n^{-k/2+1/2} \ind\{k \in 2\nats\} + \pmax^{1/2} \Bigl) \nu_n^{k-1/2}   \le 4 \nu_n^{k-1/2} \label{eq:even:boundary}
\end{align}
using $\nu_n \ge 1$ and $\pmax \le 1$.
Applying Lemma~\ref{lem:subgauss_moms}
gives the following
\begin{align}\label{eq:delt_eps_rgwth}
    \ex(\Delta_{im}^\varepsilon)^r &\leq \bigl(4
    C_1
    \sigma \nu_n^{k-1/2}r^{1/2} \bigr)^r 
\end{align}
showing that $\Delta_{im}^\eps$ is sub-Gaussian with parameter $\lesssim \sigma \nu_n^{k-1/2}$. Later in Section~\ref{sec:proof:noise:upper}, we combine this with the bound on $\Delta_{im}$ to finish the proof of Theorem~\ref{thm:noise:upper:bound}.

\paragraph{On Dominant Walk Types} 
A careful review of Lemma~\ref{lem:Akij_growth} reveals that, %
when $\nu_n \gg 1$, there are two dominant walk types in the feature noise: the simple cycles / path graphs of length $k$ %
and the Dyck paths with $k/2$ edges. 
Out of all potential subgraphs constructed from a $k$-walk, these are the two subgraphs which aggregate, or ``amplify," the feature noise the most.

For a walk type to contribute the most in expectation, its subgraph must have many configurations (for $k\ll n$ this roughly translates to maximizing the number of vertices in a graph) and it must limit the number of unique edges $t$ in its subgraph, otherwise any particular subgraph realization is less likely to appear. These two conditions lead to path and tree graphs to be the most natural contenders for subgraphs of a dominant walk type.

The feature noise dominant walk type has potentially alternate behavior in $k$, in that, for $k$ even, it is determined by %
the sparsity boundary $\nu_n \sim n^{1/k}$ (obtained by setting $\nu_n^{-k/2+1/2}  \sim  \pmax^{1/2}$ in~\eqref{eq:even:boundary}). This boundary is caused by the special nature of backtracking walks that can only appear when $k\in 2\mathbb{N}$. Backtracking walks limit the number of unique edges in their subgraphs at the cost of having fewer vertices. However, if the probability of making an edge is low enough, that is $\nu_n$ is %
small enough, then theses backtracking walks will be dominant with the largest category of the backtracking walks being the Dyck paths $\bm\breve{\mathcal{N}}_{k/2}(i,i)$.

The previous argument, and much of our future analysis, hinges on the fact $\nu_n \gg 1$. In this case, subgraph multiplicity associated with increasing the number of edges $t$ can largely be discounted. In this regime, additional vertices add a factor of approximately $\nu_n/t$ to the noise where $t \leq k$ by the nature of our walks. For $\nu_n \sim 1$, it is no longer the case that adding a vertex uniformly increases the contribution across different walk types. By similar reasoning, one can see that the dominant walk type for $\nu_n \ll 1$ would simply the edge graph where $t = 1$.

\subsection{Graph Noise and Walk Sequences}
\label{sec:product:walk:analysis}
It remains to bound the graph noise, $\Delta_{im}$, for which we rely on a high-order notion of walks, \emph{walk sequences} (or walk products), which, given the various independence properties of the adjacency matrix $A$ and the node features $X$, can be used to derive tight moment inequalities.

\begin{figure}[t]
\centering
\begin{tabular}{m{.5cm} m{6.5cm}  m{4.5cm}  m{2cm} }
\toprule
& \textbf{Walk \(w\)} & \textbf{Unique Edges \([w]\)} & \textbf{Nodes \(\bracket{w}\)} \\[1ex]
\midrule
$w_1$ & 5 $\rightarrow$ 1 $\rightarrow$ 2 $\rightarrow$ 1 $\rightarrow$ 3 $\rightarrow$ 4 $\rightarrow$ 3 $\rightarrow$ 1 $\rightarrow$ 5 & 
$\{\{1,2\}, \{1,3\}, \{1,5\}, \{3,4\}\}$ & 
$\{1, 2, 3, 4, 5\}$ \\[1ex]
\midrule
$w_2$ & 5 $\rightarrow$ 1 $\rightarrow$ 3 $\rightarrow$ 1 $\rightarrow$ 5 $\rightarrow$ 3 $\rightarrow$ 5 & 
$\{\{1,3\}, \{1,5\}, \{3,5\}\}$ & 
$\{1, 3, 5\}$ \\
\bottomrule
\end{tabular}
\caption{Walks and their unique undirected edges and nodes}
\label{fig:walks:unique:edges:nodes}
\end{figure}

Let us revisit the $k$-walk $w$ as in~\eqref{eq:walk}.
For such $w$, we write 
\[
A_w \coloneqq \prod_{\ell=1}^k %
A_{i_\ell, i_{\ell+1}}
= %
\prod_{\{i_\ell, i_{\ell+1}\}\in [w]} A_{i_\ell,i_{\ell+1}}
\]
where 
$[w] = \{\{i_1,i_2\},\{i_2,i_3\},\ldots,\{i_k,i_{k+1}\}\}$ 
is the set of unique \emph{undirected} edges of $w$. The second equality follows since $A$ is a binary symmetric matrix, i.e., $A_{ij} \in \{0,1\}$ and $A_{ij} = A_{ji}$.
The number of unique undirected edges of $w$ 
is the cardinality of set $[w]$, denoted as $|[w]|$. Occasionally, we will 
need the set of  unique vertices found in $w$ 
which we denote as $\bracket{w} = \{i_{\ell}\}_{\ell=1}^{k+1}$. Figure~\ref{fig:walks:unique:edges:nodes} illustrates $[w]$ and $\bracket{w}$ for the two walks introudced in Figure~\ref{fig:walks:and:graphs}. Note that, $[w]$ is the edge set of $G(w)$, while $\bracket{w}$ is its vertex set.

\medskip
Let $\mathcal{W}\coloneqq \mathcal{W}_k(i)$ be the set of $k$-walks which start at $i$, that is,
\[
\mathcal{W} \coloneqq \mathcal{W}_k(i) = \{w %
\; \text{as in~\eqref{eq:walk} with}\;
i_1 = i\}.
\]
With the above notation, we have $\sum_j (A^k)_{ij} = \sum_{w\in \mathcal{W}} A_w$. Let $\proj:\mathcal{W}\to [n]$ be the projection giving the last vertex of a walk, %
that is, for $w$ as in~\eqref{eq:walk},
$\proj(w) = i_{k+1}$.
Then
\[
\Delta_{im} = \sum_{w\in\mathcal{W}}(A_w - \ex[A_w]) (x_{\proj(w)})_m.
\]

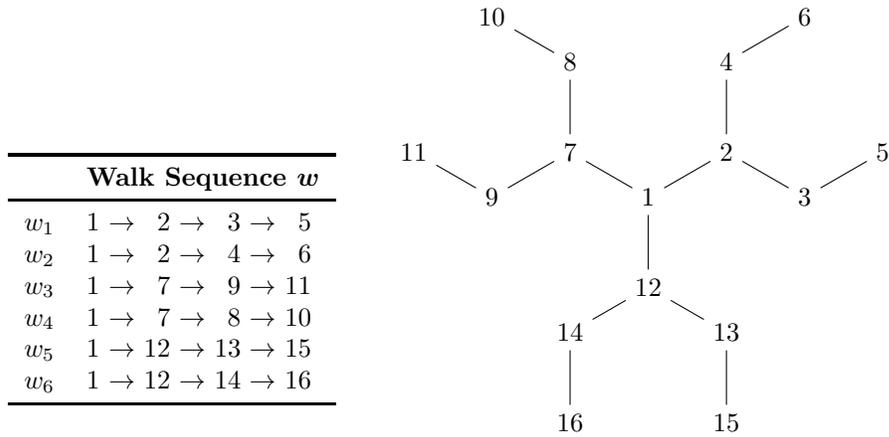
\begin{figure}[t]
\centering
\begin{tabular}{m{.3\linewidth} m{.3\linewidth}}
  \begin{tabular}{r l}
    \toprule
    & \textbf{Walk Sequence \(\bm{w}\)} \\
    \midrule
    $w_1$ & 1 $\rightarrow$ \phantom{0}2 $\rightarrow$ \phantom{0}3 $\rightarrow$ \phantom{0}5 \\
    $w_2$ & 1 $\rightarrow$ \phantom{0}2 $\rightarrow$ \phantom{0}4 $\rightarrow$ \phantom{0}6 \\
    $w_3$ & 1 $\rightarrow$ \phantom{0}7 $\rightarrow$ \phantom{0}9 $\rightarrow$ 11 \\
    $w_4$ & 1 $\rightarrow$ \phantom{0}7 $\rightarrow$ \phantom{0}8 $\rightarrow$ 10 \\
    $w_5$ & 1 $\rightarrow$ 12 $\rightarrow$ 13 $\rightarrow$ 15 \\
    $w_6$ & 1 $\rightarrow$ 12 $\rightarrow$ 14 $\rightarrow$ 16 \\
    \bottomrule
  \end{tabular}
  &
  \begin{tikzpicture}[node distance=2cm, auto]
    \newcommand{\x}{1.2cm}

    \node (n1) at (0,0) {1};

    \node (n2) at (30:\x) {2};
    \node (n3) at ([shift={(-30:\x)}]30:\x) {3};
    \node (n5) at ([shift={(-30:\x)}]30:2*\x) {5};
    \node (n4) at ([shift={(90:\x)}]30:\x) {4};
    \node (n6) at ([shift={(90:\x)}]30:2*\x) {6};

    \node (n7) at (150:\x) {7};
    \node (n9) at ([shift={(210:\x)}]150:\x) {9};
    \node (n11) at ([shift={(210:\x)}]150:2*\x) {11};
    \node (n8) at ([shift={(90:\x)}]150:\x) {8};
    \node (n10) at ([shift={(90:\x)}]150:2*\x) {10};

    \node (n12) at (270:\x) {12};
    \node (n13) at ([shift={(330:\x)}]270:\x) {13};
    \node (n15) at ([shift={(330:\x)}]270:2*\x) {15};
    \node (n14) at ([shift={(210:\x)}]270:\x) {14};
    \node (n16) at ([shift={(210:\x)}]270:2*\x) {16};

    \draw (n1) -- (n2);
    \draw (n2) -- (n3);
    \draw (n3) -- (n5);
    \draw (n2) -- (n4);
    \draw (n4) -- (n6);

    \draw (n1) -- (n7);
    \draw (n7) -- (n9);
    \draw (n9) -- (n11);
    \draw (n7) -- (n8);
    \draw (n8) -- (n10);

    \draw (n1) -- (n12);
    \draw (n12) -- (n13);
    \draw (n13) -- (n15);
    \draw (n12) -- (n14);
    \draw (n14) -- (n16);
  \end{tikzpicture}
\end{tabular}
\caption{A walk sequence $\wb$ with $r=6$ components of length $k=3$ each, and its corresponding undirected graph $G(\wb)$. This walk sequence belongs to $\Nc_{r,t,v}$ defined in~\eqref{eq:Nrtv} with $r=6$, $t = |[\wb]| = 15$ (number of unique edges) and $v = |\bracket{\wb}| = 16$ (number of unique vertices).}
\label{fig:walk:seq}
\end{figure}

Now, let $\bm{w}=(w_1,w_2,\ldots,w_r)$ be an ordered $r$-tuple of walks from $\mathcal{W}$. The set of such $r$-tuples is the $r$-fold Cartesian product $\mathcal{W}^r \coloneqq \bigotimes_{s=1}^r \mathcal{W}$. We refer to elements of $\mathcal{W}^r$ as \emph{walk sequences}. Let us also write $(\cdot)^s$ for the coordinate projection of such $r$-tuples where $\bm{w}^s = w_s$ for $s\in[r]$. In the case of multiple coordinate projection with coordinates $S=\{s_1,s_2,\ldots,s_m\}\subseteq [r]$, we preserve the tuple ordering of $\bm{w}$ such that the corresponding projection $\bm{w}^S$ satisfies
\[
s_1<s_2<\cdots<s_m\implies \bm{w}^S\coloneqq (w_{s_1},w_{s_2},\ldots,w_{s_m}).
\]
The set of 
unique undirected edges and vertices
in a walk sequence $\bm{w}$ can be computed as
\[
[\bm{w}^S] = \bigcup_{s\in S}[w_s], \qquad \bracket{\bm{w}^S} = \bigcup_{s\in S}\bracket{w_s}
\]
where, by convention, $[\bm{w}]\coloneqq [\bm{w}^{[r]}]$ and $\bracket{\bm{w}}\coloneqq \bracket{\bm{w}^{[r]}}$. Similar to the case of a single walk, we write $G(\wb)$ for the undirected graph associated with the walk sequence $\wb$, that is, the graph with vertex set $\bracket{\bm{w}}$ and edge set $[\wb]$. Figure~\ref{fig:walk:seq} shows an example of a walk sequence with $r = 6$ components each of length $k = 3$, together with its undirected graph $G(\wb)$.

With these notations, we have
\[
\Delta_{im}^r = %
 \sum_{\bm{w} \in\mathcal{W}^r} \prod_{s=1}^r (A_{\bm{w}^s}-\mathbb{E}[A_{\bm{w}^s}]) (x_{\bm \proj(\bm{w}^s)})_m.
\]
Finally, we write
\begin{align}\label{eq:rho:1:2}
\varrho_1(\bm w) \coloneqq \ex \Big[\prod_{s=1}^r (A_{\bm{w}^s}-\mathbb{E}[A_{\bm{w}^s}]) \Bigl], \quad 
\varrho_2(\bm w) := 
\ex\Bigl[ \prod_{s=1}^r (x_{\bm \proj(\bm{w}^s)})_m\Bigr]    
\end{align}
and $\varrho(\bm w) := \varrho_1(\bm w) \varrho_2(\bm w)$. Note that we are suppressing the dependence of $\varrho(\wb)$ on $i$ and $m$ for simplicity. By independence of $A$ and $x$,
\[
\mathbb{E}[\Delta_{im}^r] = \sum_{\bm{w}\in\mathcal{W}^r} 
\varrho(\bm w).
\]
For a walk sequence $\bm{w}\in\mathcal{W}^r$, let us write $A_{\bm w} \coloneqq \prod_{s=1}^rA_{\bm{w}^s}$ and similarly $ A_{\wb^S} = \prod_{s \in S} A_{\wb^s}$.
We have the following control of $\varrho(\wb)$:
\begin{lem}\label{lem:rho:control}
    $|\varrho_1(\bm w)| \le 2^r \ex[A_{\wb}] \le 2^r \pmax^{|[\wb]|}$ and
    $|\varrho_2(\bm w)| \le 
    (2\max\{C_1 \sigma r^{1/2},\,%
    \infnorm{\mu_{m*}}
    \})^r$.
\end{lem}

\subsection{Leading Order Walk Decomposition}

For the walk analysis, we are interested in walk types which contribute most to %
 $\ex[\Delta_{im}^r]$. %
To this end, we make use of the partial centering found in the walk products $\varrho_1(\bm{w})$. To accomplish this, we partition walks $\bm{w}^s$ within a walk sequence $\bm{w}$ according to their overlapped edges. This partition approach is similar to the edge partitions introduced in~\cite{Erdos13}, where it was necessary to control the moments of a related deviation term, $\frac{1}{n}\mathbb{E}\big[\mathbbm{1}_{[n]}^T(A- \mathbb{E}[A])^r\mathbbm{1}_{[n]}\big]$.

For every walk sequence $\bm{w}\in\mathcal{W}^r$, we define a partition $\Gamma(\bm{w})$ on $[r]$ by declaring $s,s'\in[r]$ to be equivalent if and only if $\bm{w}^s$ and $\bm{w}^{s'}$ share an undirected edge, that is, $[\bm{w}^s]\cap[\bm{w}^{s'}]\neq \varnothing$. Indexing the equivalences classes in $\Gamma(\bm{w})$ as $\Gamma_q(\bm{w})\subseteq [r]$, we have $[r] = \bigsqcup_q \Gamma_q(\bm{w})$. We will refer to %
$\Gamma(\bm{w})$
as the $[r]$-partition of $\bm w$. We say $\bm w \in\mathcal{W}^r$ is \emph{non-overlapping} if $|\Gamma_q(\bm{w})|=1$ for some $q$, otherwise, it is \emph{overlapping}. We occasionally treat $\Gamma(\wb)$ as an ordered tuple, by ordering the partition components $\Gamma_q(\wb)$ according to their smallest element, as in the example below:

\begin{exa}\label{exa:wb:parition}
Walk sequence $\wb$ in Figure~\ref{fig:walk:seq} is overlapping with $[6]$-partition $\Gamma(\wb) = (\{1,2\},\{3,4\},\{5,6\})$, meaning that walks, e.g.,  $\wb^3$ and $\wb^4$ share at least an edge (in this case $\{1,7\}$), while walks $\wb^3$ and $\wb^6$ share no edge, and so on. We can refer to the second component of the partition which is $\Gamma_2(\wb) = \{3,4\}$, since $\{3,4\}$ has the second smallest element among the three components. 
\end{exa}

Given the independence of $A$ and $X$, we only need to consider overlapping $\bm{w}$, for which $|\Gamma_q(\bm w)| \geq 2$ for all $q$:
\begin{lem}\label{lem:overlap_rho}
    $\varrho(\bm w) = 0$ if $\bm w$ is non-overlapping.
\end{lem}
Seen simply, Lemma~\ref{lem:overlap_rho} is stating that any walk $\bm{w}^s$ which does not overlap with any other walks in the walk sequence %
$\bm{w}$ will factor out and evaluate to zero in our moment calculation of $\ex[\Delta_{im}]$.
For overlapping walk sequences, we have the following control:
\begin{lem}\label{lem:equiv_partition}
    Let $\Gamma(\bm{w}) = \{\Gamma_q\}_{q=1}^Q$ for some overlapping $\bm w$. Then, $Q \leq \lfloor r/2 \rfloor$ and 
    \[
    |[\bm{w}^{\Gamma_q}]| \leq |\Gamma_q|(k-1) + 1,\qquad\forall q\in[Q]
    \]
    and hence $|[\bm w]|\leq rk - \lceil r/2\rceil$.
\end{lem}

To quantify the growth of $\ex[\Delta_{im}^r]$, we partition overlapping walk sequences based on their number of unique edges, $t$, and unique vertices, $v$:
\begin{align}\label{eq:Nrtv}
\Nc_{r,t,v} \coloneqq \bigl\{\;\bm{w} \in\mathcal{W}^r:\, |[\bm{w}]| = t,\; |\bracket{\bm w}|=v,\; \bm{w}\;\text{is overlapping}\;\bigr\},
\end{align}
The following lemma bounds the size of $\Nc_{r,t,v}$:
\begin{lem}\label{lem:counting}
    We have $|\mathcal{N}_{r,t,v}|\leq (v-1)^{rk}\binom{n-1}{v-1}$ and hence for $b \le t$,
    $
    \sum_{v = 1}^{b+1}|\mathcal{N}_{r,t,v}| \le b^{rk - b} (en)^b.
    $
\end{lem}
The overlapping and rooted nature of the walks $\bm{w}^s$ give us a few consequences. 
Let 
\begin{align}
t_* \coloneqq r(k-1/2).
\end{align}
Then, 
by Lemma~\ref{lem:equiv_partition}, for $r\in 2\nats$, we have $
\mathcal{N}_{r,t,v}=\varnothing$ for $t > t_*$, that is, $t_*$ is the maximum number of unique edges for an overlapping $r$-sequence of $k$-walks. The walk sequence in Figure~\ref{fig:walk:seq}, for example, achieves this maximum since $t = t_* = 6(3-1/2) = 15$ in this case.
Additionally since each walk in the walk sequence $\bm w$ starts at $i$, the unique edges found in $\bm{w}$ correspond to a connected graph, meaning $\mathcal{N}_{r,t,v}=\varnothing$ for $v > t+1$.

Consequently, the $r$th moment of $\Delta_{im}$ can be decomposed as
\begin{align*}
	\mathbb{E}[\Delta_{im}^r] &= \sum_{t=1}^{t_*}\sum_{v=1}^{t+1}\sum_{\bm{w} \in \mathcal{N}_{r,t,v}} \varrho(\bm w) = \Thi_{im}(r) + \Tlo_{im}(r)
\end{align*}
where 
\begin{align}
	\Thi_{im}(r)
 \coloneqq \sum_{\bm{w} \in \mathcal{N}_{r,t_*,(t_*+1)}}  \varrho(\bm w), \qquad 
\Tlo_{im}(r)
 \coloneqq \sum_{t=1}^{t_*}\sum_{v=1}^{t+1}\sum_{\bm{w} \in \mathcal{N}_{r,t,v}}\varrho(\bm w)\,1\{v\neq t^*+1\}.
\label{eq:T:defs}
\end{align}
When $i,m$ and $r$ are fixed, we often drop the dependence on them and simply write $\Thi$ and $\Tlo$.
For brevity, we also write 
\[
\mathcal{N}_* \coloneqq \mathcal{N}_{r,t_*,(t_*+1)}.
\]
Terms $\Thi$ and $\Tlo$ can be upper-bounded using the moment contributions $\varrho(\bm w)$ of walk sequences $\bm w$ which have the largest number of unique vertices $v$. This leads to the following key result, by combining Lemmas~\ref{lem:rho:control} and~\ref{lem:counting}:
\begin{lem}\label{lem:T_grwth}
    Let $\kappa_{0,m} = 4 \max\{\frac{C_1 \sigma}{\infnorm{\mu_{m*}}}, 1\}$ where $C_1$ is the constant in Lemma~\ref{lem:subgauss_moms} and $\kappa_0 = \max_m \kappa_{0,m}$. Then for any even $r$ such that $\kappa_0^r\, t_*^{r} e^{t_*} \le \frac13 \nu_n^{1-\epsilon}$, we have
    \begin{align}
        |\Thi_{im}(r)| &\;\leq\; (\sqrt{r}\,\infnorm{\mu_{m*}})^r \, \nu_n^{t_*}, \label{eq:Thi:bnd}\\
         |\Tlo_{im}(r)| &\;\leq\; (\sqrt{r} \,\infnorm{\mu_{m*}} )^{r} \cdot 
         \nu_n^{t_* - \epsilon}. %
    \end{align}
\end{lem}
Note that since $t_* < rk$, the condition of Lemma~\ref{lem:T_grwth} is satisfied for all $r \le r_n(\epsilon)$ as defined in~\eqref{eq:rn:def}.
We are now ready to prove the high-probability noise upperbound for $\dev$.

\subsection{Proof of Theorem~\ref{thm:noise:upper:bound}}
\label{sec:proof:noise:upper}

    Let $r \in 2\nats$.
    Combining~\eqref{eq:dev:upper:jensen} and~\eqref{eq:dev:delta:decomp}, we have
    \begin{align*}
        \ex(\dev)^r \leq \frac{d^{r/2-1}}{n} \sum_{i,m} 2^{r-1} \bigl(\ex [\Delta_{im}^r] + \ex(\Delta_{im}^\varepsilon)^r\bigr).
    \end{align*}
    Bounding $\infnorm{\mu_{m*}}\le \maxnorm{\mu}$, for the first term, we obtain
    \[
    \ex [\Delta_{im}^r] \;\le\; |\Tlo_{im}(r)|   + |\Thi_{im}(r)| %
    \;\le\;  2(\sqrt{r}\,\maxnorm{\mu})^r \, \nu_n^{t_*}  
    \;=\;
     2(\sqrt{r}\,\maxnorm{\mu} \nu_n^{k-1/2})^r 
    \]
     using $t_* = rk - r/2$. 
    For the second term,~\eqref{eq:delt_eps_rgwth} gives 
    \[
    \ex(\Delta_{im}^\varepsilon)^r 
    \leq \bigl(4 C_1 \sigma \sqrt{r\pmax} \nu_n^{k-1/2} \bigr)^r
    \]
   Let $\kappa_3 = \max\{ 4 C_1 \sigma, \maxnorm{\mu}\}$. Then,
   $\bigl(\ex [\Delta_{im}^r] + \ex(\Delta_{im}^\varepsilon)^r\bigr) \le 3 (\kappa_3 \sqrt{r} \nu_n^{k-1/2})^r$. It follows that
   \[
    \ex(\dev)^r \le \frac32 d^{r/2} (\kappa_3 \sqrt{r} \nu_n^{k-1/2})^r \le (\kappa_3 \sqrt{2dr} \nu_n^{k-1/2})^r
   \]
   using $3/2 \le (\sqrt{2})^r$ for all even $r \le r_n$. Applying Lemma~\ref{lem:moment:concent2} with $K = \kappa_3 \nu_n^{k-1/2}$, $\eta = 1/2$, $C = 4d$, the result follows.

\subsection[Characterizing N*]{Characterizing $\mathcal{N}_*$}\label{sec:char_Nrt}
The rest of Section~\ref{sec:noise} is devoted to proving the noise lowerbound (Theorem~\ref{thm:noise:lower:bound}).
To obtain a sharp lowerbound, 
we construct
a sufficiently tight proxy $\Thit$ to $\Thi$ (Section~\ref{sec:proxy:for:Thi}), one that satisfies (Section~\ref{sec:Thi:Thit:dev})
\[
|\Thi - \Thit|\lesssim \nu_n^{t_*-1}.
\]
That is, the discrepancy between $\Thit$ and $\Thi$ is of a lower order than $\nu_n^{t_*}$, which we know is the upper bound on $\Thi$ form~\eqref{eq:Thi:bnd}. Then, we establish a lower bound on $\Thit$ of the order $\nu_n^{t_*}$ (Section~\ref{sec:Thit:lower:bound}), which by the (reverse) triangle inequality implies the same lower bound on $\Thi$ up to constants, establishing that $\Thi$ is indeed tightly concentrated from above and below around $\nu_n^{t_*}$.
This proxy approach is similar to the one used in Section~\ref{sec:signal}.

To execute the above plan,
we first investigate the structure of walk sequences in $\mathcal{N}_*$. We refer to the element of $\Nc_*$ as \emph{maximal walk sequences}.
Let $G(\bm w) = (\bracket{\bm w},\, [\bm w])$ be the undirected graph associated with the walk sequence $\bm w$. The following result provides a complete characterization of $G(\wb)$ for walks sequences in $\mathcal{N}_{*}$ as well as the associated $[r]$-partition $\Gamma(\wb)$. An $i$-rooted tree, is a rooted tree with root node $i$. Complete graph on $[r]$ is denoted $K_r$.

\begin{lem}[Structure of $\Nc^*$]\label{lem:tree:character}
       Let $\bm{w} \in \mathcal{N}_{*}$ and $\Gamma(\wb) = \{ \Gamma_q \}_{q=1}^Q$. Then, the following hold:
       \begin{enumerate}[(a)]
        \item $\Gamma(\wb)$ is a perfect matching on $[n]$. That is, $|\Gamma_q| = 2$ for all $q$, hence $Q = r/2$.
           \item  $G(\bm w)$ is an $i$-rooted tree.
           \item  $G(\wb^{\Gamma_q}), q \in [Q]$ are $i$-rooted subtrees of  $G(\wb)$; they are vertex disjoint except at the root.
           \item For $\Gamma_q = \{s,s'\}$, $G(\wb^{\Gamma_q})$ consists of two $i$-rooted subtrees $G(\wb^s)$ and $G(\wb^{s'})$ that share the same first edge $(i,j_s)$, but are otherwise disjoint.
       \end{enumerate}
       Moreover, $\Xi_r = \{\Gamma(\bm w):\, %
       \bm{w}\in \mathcal{N}_{*}
       \}
       $ is the set of all 
       perfect matchings on $K_r$,
       hence $|\Xi_r| = (r-1)!!$.
\end{lem}

\newcommand\jb{\bm j}
Let us introduce the following terminology: For the two subtrees $G(\wb^s)$ and $G(\wb^{s'})$ in part~(d) of Lemma~\ref{lem:tree:character}, we refer to the disjoint parts of the $\wb^s$ and $\wb^{s'}$, after the initial edge, as a \emph{matched pair of emanating branches}. Thus $G(\wb)$ can be described as follows: An $((r/2) + 1)$-vertex star centered on $i$, which we refer to as the \emph{core star}, to each of its $r/2$ leaves is attached a matched pair of emanating branches, each of length $k-1$, and mutually disjoint except at their root.  The non-matched emanating branches (i.e., those attached to different leaves of the core star) are completely disjoint.

Fix $\bm{w} \in \mathcal{N}_{*}$ and let 
$\Gamma(\wb) = \Gamma = ( \Gamma_q )_{q=1}^Q$.
Recall that, in general, $\Gamma$ forms a partition of $[r]$. By Lemma~\ref{lem:tree:character}(a), each $\Gamma_q$ is of the form $\{s,s'\}$ for $s \neq s'$.  Then, if $(i,j_s)$ and $(i,j_{s'})$ are the first edges of $\wb^s$ and $\wb^{s'}$, by Lemma~\ref{lem:tree:character}(d), we will have $j_s = j_{s'}$ and that is the only overlap among the two walks. 
Let us write $j(\Gamma_q)$ for this common endpoint ($j_s = j_{s'}$) of the first edge of the two walks $\wb^u, u \in \Gamma_q$. We also write $\jb(\Gamma) \in [n]^Q$ for the vector whose $q$th coordinate is $j(\Gamma_q)$. We denote this $q$th coordinate with a superscript, that is, $\jb^q(\Gamma) = j(\Gamma_q)$. In short, $\jb(\Gamma)$ collects endpoints of the edges of the core star of $G(\wb)$, one for each matched pair in $\Gamma$. By Lemma~\ref{lem:tree:character}, $Q = r/2$ hence $\jb$ is $r/2$-dimensional.

\begin{exa}
To illustrate, note that $\wb$ in Figure~\ref{fig:walk:seq} is a maximal walk sequence, belonging to $\Nc_* = \Nc_{6,15,16}$ for $r=6$ and $k=3$. The core store is the subgraph on nodes $\{1,2,7,12\}$. The matched pair of emanating branches attached to, say $7$ are $7-8-10$ and $7-9-11$, which do not overlap. Similarly, the matched pair of emanating branches attached to $2$ are $2-4-6$ and $2-3-6$ and so on. As shown in Example~\ref{exa:wb:parition}, $\Gamma(\wb) = \Gamma \coloneqq (\{1,2\},\{3,4\},\{5,6\})$ is indeed a %
perfect matching on $K_6$.
We have $j(\Gamma_3) = j(\{5,6\}) = 12$ and $\jb(\Gamma) = (2,7,12)$. Similarly, $\jb^3(\Gamma) = 12$, just the third coordinate of $\jb(\Gamma)$.   
\end{exa}

\begin{figure}
    \centering
    \begin{tikzpicture}[node distance=2cm, auto]
      \newcommand{\x}{1.2cm}
    
      \node (n1) at (0,0) [circle, fill=red, inner sep=2pt] {};
    
      \node (n2) at (30:\x) [circle, fill=red, inner sep=2pt] {};
      \node (n3) at ([shift={(-30:\x)}]30:\x) [circle, fill=black, inner sep=2pt] {};
      \node (n5) at ([shift={(-30:\x)}]30:2*\x) [circle, fill=black, inner sep=2pt] {};
      \node (n4) at ([shift={(90:\x)}]30:\x) [circle, fill=black, inner sep=2pt] {};
      \node (n6) at ([shift={(90:\x)}]30:2*\x) [circle, fill=black, inner sep=2pt] {};
    
      \node (n7) at (150:\x) [circle, fill=red, inner sep=2pt] {};
      \node (n9) at ([shift={(210:\x)}]150:\x) [circle, fill=black, inner sep=2pt] {};
      \node (n11) at ([shift={(210:\x)}]150:2*\x) [circle, fill=black, inner sep=2pt] {};
      \node (n8) at ([shift={(90:\x)}]150:\x) [circle, fill=black, inner sep=2pt] {};
      \node (n10) at ([shift={(90:\x)}]150:2*\x) [circle, fill=black, inner sep=2pt] {};
    
      \node (n12) at (270:\x) [circle, fill=red, inner sep=2pt] {};
      \node (n13) at ([shift={(330:\x)}]270:\x) [circle, fill=black, inner sep=2pt] {};
      \node (n15) at ([shift={(330:\x)}]270:2*\x) [circle, fill=black, inner sep=2pt] {};
      \node (n14) at ([shift={(210:\x)}]270:\x) [circle, fill=black, inner sep=2pt] {};
      \node (n16) at ([shift={(210:\x)}]270:2*\x) [circle, fill=black, inner sep=2pt] {};
    
      \draw[red] (n1) -- (n2);
      \draw (n2) -- (n3);
      \draw (n3) -- (n5);
      \draw (n2) -- (n4);
      \draw (n4) -- (n6);
    
      \draw[red] (n1) -- (n7);
      \draw (n7) -- (n9);
      \draw (n9) -- (n11);
      \draw (n7) -- (n8);
      \draw (n8) -- (n10);
    
      \draw[red] (n1) -- (n12);
      \draw (n12) -- (n13);
      \draw (n13) -- (n15);
      \draw (n12) -- (n14);
      \draw (n14) -- (n16);
    \end{tikzpicture}
    \caption{The unlabeled graph $G^*$ that all $G(\wb), \wb \in \Nc_*$ are isomorphic to, for $r=6$ and $k=4$. Its core star is colored red, while matched pairs of emanating branches are in black. }
    \label{fig:G:star}
\end{figure}
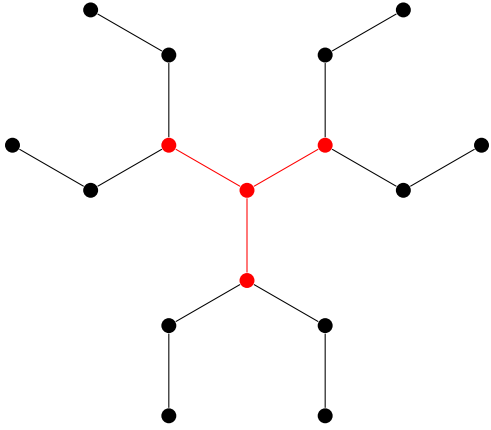

As a consequence of Lemma~\ref{lem:tree:character},  the graphs $G(\wb), \wb \in \Nc_*$ are all isomorphic to a single unlabeled graph, which we denote as $G^*$. This allows us to determine the size of $\Nc^*$ exactly (see the proof of Lemma~\ref{lem:T_grwth}). Figure~\ref{fig:G:star} illustrates $G^*$ for the case $r=6$ and $k = 4$, with the core star colored in red. This is essentially the unlabeled version of the graph shown in Figure~\ref{fig:walk:seq}. Every other graph $G(\wb), \wb \in \Nc_*$ can be obtained by assigning $t_*+1 = 16$ elements from $[n]$ to the vertices of $G^*$ in Figure~\ref{fig:G:star}, and specifying a %
matching on $K_r$
to map the paired branches to constituent walks $\wb^s$ (i.e., which of the $r$ walks in the sequence gets assigned to each branch). This matching can  be equivalently seen as a branch coloring on the graph $G(\wb)$, where branches of $G$ are colored according to the order $(\wb^1,\,\wb^2,\ldots,\,\wb^r)$ in the walk sequence $\wb$.  In this example, $n \ge 16$, but otherwise unspecified.

\paragraph{Factorizing $\rho(\wb)$} Lemma~\ref{lem:tree:character} can be used to completely factorize $\rho(\bm w)$ for even $r$. Let 
\begin{align}
    \bmu \coloneqq (\mu_{y_i}, i \in [n]),
\end{align}
which we view as a $d \times n$ matrix and let $\bmu_m$ be its $m$th row, viewed as a column vector, and $\bmu_{mi} = (\mu_{y_i})_m$ be the $i$th coordinate of $\bmu_m$. Note that this definition for $M$ agrees with definition~\eqref{eq:mut:averaged} found in the signal analysis of Section~\ref{sec:signal}.
\begin{cor}\label{cor:w_factor}
     Let $\bm{w} \in \mathcal{N}_{*}$, $\Gamma = \Gamma(\wb)$ and $\jb = \jb(\Gamma)$.  Then, \begin{align}\label{eq:varrho_Gamma}
         \varrho(\bm w) %
         &= \prod_{q=1}^{r/2} \varrho(\bm w^{\Gamma_q})
     \end{align}
     which for $\bm w^{\Gamma_q} = (\bm v^1,\bm v^2)$,
     \begin{equation}\label{eq:rho:wb:Gammaq}
         \varrho(\bm w^{\Gamma_q})=  p_{i\bm{j}^q}(1-p_{i\bm{j}^q})\prod_{\alpha=1}^2\Big(
      \bmu_{m \proj(\bm{v}^\alpha)} \prod_{\ell=2}^kp_{(\bm{v}^\alpha)_\ell}\Big),
     \end{equation}
     where $(\bm v^\alpha)_\ell$ is the $\ell$th edge in walk $\bm v^\alpha$ and $p_{(\bm v^\alpha)_\ell} = p_{i_\ell j_\ell}$ if $(\bm v^\alpha)_\ell = (i_\ell, j_\ell)$.  
\end{cor}

Corollary~\ref{cor:w_factor} has no dependence on $\sigma$. Noise variance $\sigma$ appears in moments of $\ex[\eps_j^m]$ for $m > 1$, which is only possible for endpoint intersections on $\bm{w}$, that is $\proj(\bm{w}^s) = \proj(\bm{w}^{s'})$ for $s\neq s'$. The star shape of walks $\bm{w}\in\mathcal{N}_*$ guarantee $\proj(\bm{w}^s)\neq \proj(\bm{w}^{s'})$ for all $s\neq s'$.

When viewed as a function of $\mathcal{W}^2_k$, we express the RHS of~\eqref{eq:rho:wb:Gammaq} as $\varrho^*(\bm{v}):\mathcal{W}^2_k\to\mathbb{R}$ where
\begin{equation}\label{eq:varrho_star}
    \varrho^*(\bm{v}) = p_{i\bm{j}^q}(1-p_{i\bm{j}^q})\prod_{\alpha=1}^2\Big(
      \bmu_{m \proj(\bm{v}^\alpha)} \prod_{\ell=2}^kp_{(\bm{v}^\alpha)_\ell}\Big).
\end{equation}
With some extra precaution, $\varrho^*$ can be extended according to~\eqref{eq:varrho_Gamma} for inputs in $\mathcal{W}^{r}_k$. That is for $\bm{w}\in\mathcal{W}^{r}_k$ and for partition $\Gamma\in \Xi_r$, not necessarily equal to $\Gamma(\bm{w})$, we have
\begin{equation}\label{eq:varrho_star_gamma}
    \varrho^*(\bm{w};\Gamma) \coloneqq \prod_{q=1}^{r/2} \varrho^*(\bm{w}^{\Gamma_q}).
\end{equation}
Note that we have suppressed the dependence of $\varrho^*(\bm{w};\Gamma)$ on $i$ and $m$ for simplicity.

\subsection[Proxy for Thi]{Proxy for $\Thi$}
\label{sec:proxy:for:Thi}

From now on, let $\Gamma$ be an element of 
$\Xi_r$, that is, a perfect matching on $[r]$. 
Let
\[
\Nc_*(\Gamma, \jb) := \{ \wb \in \Nc_*:\; \Gamma(\wb) = \Gamma,\; \jb(\Gamma) = \jb\},
\]
the collection of walk sequences in $\Nc_*$ that have the same matching of walks, and the same elements of $[n]\setminus\{i\}$ as the leaves of the core star. Let $\tuples{[n]\setminus\{i\}}{r/2}$ be the set of 
all ordered $(r/2)$-tuples of distinct elements from $[n]\setminus\{i\}$. These represent all the possibilities for the leaves of the core star. We have
\begin{equation}
    \mathcal{N}_* = \bigsqcup_{\Gamma \in \Xi_r}\bigsqcup_{\bm j \in \tuples{[n]\setminus\{i\}}{r/2}} \mathcal{N}_*(\Gamma,\bm{j}),
\end{equation}
Consider the following collection of $k$-walk sequences of length $r$
\begin{align}\label{eq:Wkr_j}
    \mathcal{W}_k^r(\Gamma, \bm{j}) \coloneqq 
    \big\{(\wb^1,\wb^2,\dots,\wb^r)\; &\text{such that for all $q \in [r/2]$ and $s \in \Gamma_q$} \notag \\
    &\text{$\wb^s$ is a $k$-walk whose first edge is $(i,\jb^q)$}\}.  
\end{align}
The walk sequences in $ \mathcal{W}_k^r(\Gamma, \bm{j})$ are constructed almost similarly to those in $\Nc_*(\Gamma, \jb)$ except that the paired emanating branches attached to the core star are allowed to intersect. Now let 
\begin{equation}\label{eq:T:hi:approx}
    \Thit_{im}(r) \coloneqq \sum_{\Gamma \in \Xi_r} \sum_{\bm j \in \tuples{[n]\setminus\{i\}}{r/2}}\sum_{\bm w\in\mathcal{W}_k^{r}(\Gamma, \bm j)}\varrho^*(\wb;\Gamma).
\end{equation}
Note that replacing $\mathcal{W}_k^{r}(\Gamma, \bm j)$ in the inner sum above with $\Nc_*(\Gamma, \jb)$ would give us $\Thi$ exactly. In fact, since $\Thi$ and $\Thit$ utilize the same tree decomposition, their summands, $\varrho(\bm{w})$ and $\varrho^*(\wb;\Gamma)$ respectively, can be upperbounded identical fashion:
\begin{lem}\label{lem:max:rho:star}
    We have $|\varrho^*(\wb;\Gamma)| \;\leq\; 
        \infnorm{\mu_{m*}}^r \, \pmax^{t_*}$ for all $ \wb \in \walk_k^r(\Gamma,\bm j)$. The same bound holds for $|\varrho(\wb)|$ for any $\wb \in \Nc_*$.
\end{lem}
\begin{proof}
    From~\eqref{eq:varrho_star_gamma}, we have  %
    \begin{align*}
      |\varrho^*(\wb;\Gamma)| \;\leq\; 
        \bigl(\infnorm{\mu_{m*}}^2 (\pmax)^{2(k-1) + 1} \bigr)^{r/2} = 
        \infnorm{\mu_{m*}}^r \, \pmax^{rk-r/2}
    \end{align*}
which is the desired result. The second assertion follows from the first by noting that $\rho(\wb) = \varrho^*(\wb; \Gamma(\wb))$ for any $\wb \in \Nc_*$.
\end{proof}

The expression  for $\Thit$ can be further simplified. We have by definition~\eqref{eq:varrho_star_gamma}
\[
\sum_{\bm w \in \walk_k^r(\Gamma,\bm j)} \varrho^*(\wb;\Gamma) = \sum_{\wb} \prod_q \varrho^*(\wb^{\Gamma^q}) = \prod_q \sum_{\wb^{\Gamma_q}} \varrho^*(\wb^{\Gamma_q}),
\]
where the inner sum ranges over a pair of walks of the form $\bm w^{\Gamma_q} = (\bm v^1,\bm v^2)$ with $\bm v^\alpha = (i, \jb^q) \oblong \widetilde{\bm v}^\alpha$ for $\alpha=1,2$  where $\oblong$ denotes a walk concatenation, and each  $\widetilde{\bm v}^\alpha$ is a walk of length $k-1$ starting at $\jb^q$, that is, $\widetilde{\bm v}^\alpha \in \walk_{k-1}(\jb^q)$. Using~\eqref{eq:varrho_star} we obtain
\begin{align*}
    \sum_{\wb^{\Gamma_q}} \varrho^*(\wb^{\Gamma_q}) &=  p_{i\bm{j}^q}(1-p_{i\bm{j}^q})
    \sum_{\widetilde{\bm v}^1, \widetilde{\bm v}^2 \in \walk_{k-1}(\jb^q)}
    \prod_{\alpha=1}^2\Big(
      \bmu_{m \proj(\widetilde{\bm{v}}^\alpha)} \prod_{\ell=2}^kp_{(\widetilde{\bm{v}}^\alpha)_\ell}\Big),
\end{align*}
since $\bm v^\alpha$ and $\widetilde{\bm v}^\alpha$ are identical after the first edge. The sum factorizes into two identical terms, hence equal to
\[
 \bigg( \sum_{\widetilde{\bm v} \in  \walk_{k-1}(\jb^q)}
      \bmu_{m \proj(\widetilde{\bm{v}}^\alpha)} \prod_{\ell=2}^kp_{(\widetilde{\bm{v}}^\alpha)_\ell}\bigg)^2 = \bigl(e_{\bm{j}^q}^T \mathbb{E}[A]^{k-1} \bmu_m\bigr)^2.
\]
Putting the pieces together, we have
\begin{align}
    \sum_{\bm w\, \in\, \walk_k^r(\Gamma,\bm j)} \varrho^*(\wb;\Gamma) = \prod_{q=1}^{r/2} p_{i\bm{j}^q}(1-p_{i\bm{j}^q})\big(e_{\bm{j}^q}^T \mathbb{E}[A]^{k-1} \bmu_m\big)^2,
    \label{eq:identity_pushing}
\end{align}
showing that the inner sum in~\eqref{eq:T:hi:approx} does not depend on $\Gamma$. This leads to the following simplified expression
\begin{align}\label{eq:T:hi:approx:simple}
    \frac{\Thit} {(r-1)!!}= \sum_{\bm j \in \tuples{[n]\setminus\{i\}}{r/2}}
    \prod_{q=1}^{r/2} p_{i\bm{j}^q}(1-p_{i\bm{j}^q})\big(e_{\bm{j}^q}^T \mathbb{E}[A]^{k-1} \bmu_m\big)^2.
\end{align}

\subsection[Controlling |Thi - Thit|]{Controlling $| \Thi - \Thit|$}
\label{sec:Thi:Thit:dev}
We will show that the difference $|\Thi - \Thit|$ is of lower order, hence $\Thit$ will be a good surrogate for $\Thi$ to analyze for the lower bound.

By our earlier walk construction we have $\mathcal{N}_*(\Gamma, \bm{j})\subset \mathcal{W}_{k}^r(\Gamma, \jb)$. %
The difference is small relative to the larger set:
\begin{lem}\label{lem:Nstar:Gamma:W:Gamma}
    Let $\mathcal{N}_*(\Gamma, \bm{j})$ be defined as above. Then, for all $n\geq t_* = r(k-1/2)$
    \[
    \frac{|\mathcal{N}_*(\Gamma, \bm{j})|}{|\mathcal{W}_{k}^r(\Gamma, \jb)|} \geq 
    1- \frac{t_*^2}{n},
    \]
    and consequently, $|\mathcal{W}_k^{r}(\Gamma, \bm j)\setminus \mathcal{N}_*(\Gamma, \bm j)|  \le t_*^2\, n^{t_*-r/2-1}$.
\end{lem}

We are now ready to show that $|\Thi - \Thit|$ is of lower order:
\begin{lem}\label{lem:Thi_proxy_dev}
    For any 
    $\bm{j}\in \tuples{[n]\setminus\{i\}}{r/2}$
    and $n \geq rk-r+1$,
    \[
    |\Thi_{im}(r) - \Thit_{im}(r)| \leq (r-1)!!\, t_*^2\, \infnorm{\mu_{m*}}^r\, p_{\rm max}\,\nu_n^{t_*-1}.
    \]
\end{lem}
\begin{proof}
    Let $\mathcal R(\Gamma, \jb) = \mathcal{W}_k^r(\Gamma,\jb)  \setminus \mathcal{N}_*(\Gamma,\jb)$.
    Factorizing $\Thi$ into a similar form as $\Thit$,
    \begin{align*}
        \Thi \;=\; \sum_{\bm w\in\mathcal{N}_*} \varrho(\bm w)
        \;=\; \sum_{\Gamma \in \Xi_r} \sum_{\bm j \in \tuples{[n]\setminus\{i\}}{r/2}}\sum_{\bm w\in\mathcal{N}_*(\Gamma, \bm j)} \varrho^*(\bm w;\Gamma),
    \end{align*}
    where  $\varrho(\wb) = \varrho^*(\wb;\Gamma)$ for $\wb \in \Nc^*(\Gamma, \jb)$ by Corollary~\ref{cor:w_factor}.
    By the inclusion $\mathcal{N}_*(\Gamma,\jb) \subset \mathcal{W}_k^r(\Gamma,\jb)$, %
    \begin{align*}
        \Thit - \Thi =  \sum_{\Gamma \in \Xi_r} \sum_{\bm j \in \tuples{[n]\setminus\{i\}}{r/2}}\sum_{\bm w\in \mathcal R(\Gamma,\jb)}\varrho^*(\bm w;\Gamma).
    \end{align*}
    Then, we have 
    \begin{align*}
           |\Thi - \Thit| &\leq |\Xi_r| \cdot |\tuples{[n]\setminus\{i\}}{r/2}| \cdot \max_{\Gamma, \jb} \Bigl\{ |\mathcal R(\Gamma,\jb)| \cdot \max_{\wb \in \mathcal R(\Gamma,\jb)} |\varrho^*(\wb;\Gamma)|\Bigr\}\\
        &\leq (r-1)!! \cdot \big(n^{r/2}\big) \cdot t_*^2\, 
        n^{t_*-r/2-1} \cdot \infnorm{\mu_{m*}}^r\,p_{\rm max}^{t_*}
    \end{align*}
    using Lemma~\ref{lem:Nstar:Gamma:W:Gamma} and Lemma~\ref{lem:max:rho:star}. With $\nu_n = np_{\rm max}$,  %
    the last bound is the claimed bound.
\end{proof}

\subsection[Thit Lowerbound]{$\Thit$ Lowerbound}
\label{sec:Thit:lower:bound}

We prove a $\Thit$ lowerbound for the specific case of $r=2$, for which $t_* = 2k-1$. We note that this proof can be straightfowardly extended for the case of $r\in 2\nats$. 

Let 
$e_j$ be the $j$th standard basis vector of $\reals^n$,  and for $\ell \in [L]$, define
\begin{align}\label{eq:Yell:def}
    \Imat_\ell \coloneqq \sum_{j: y_j \,\in \,\mathcal I_\ell} e_j e_j^T.
\end{align}
Note that for any matrix $H \in \reals^{n \times n}$, we have 
 $ \fnorm{H \Imat_\ell}^2 = \sum_{j: y_j \in \mathcal I_\ell} \norm{H e_j}_2^2$.

\begin{lem}\label{lem:Thit:lower}
    Assume~\ref{as:scale}--\ref{as:nu:upper} and suppose $y_i = \ell$. Then we have
    \[
    \sum_m 
    \Thit_{im}(2)
    \;\ge\; c_B c_\nu\,  \Bigl(  \frac{\nu_n}{n} \fnorm{ \bmu \,\ex[A]^{k-1} \Imat_\ell}^2 - d C_\mu^2%
    \frac{\nu_n^{2k-1}}n \Bigr).
    \]
\end{lem}
\begin{proof}
    Applying definition~\eqref{eq:T:hi:approx} of $\Thit_{im}(2)$,
    \begin{align*}
        \Thit_{im}(2) &= \sum_{\Gamma\in\Xi_{2}}\sum_{\jb \in \tuples{[n]\setminus\{i\}}{1}}\prod_{q=1}^{r/2} p_{i\bm{j}^q}(1-p_{i\bm{j}^q})\big(e_{\bm{j}^q}^T \mathbb{E}[A]^{k-1} \bmu_m\big)^2\\
        &=\sum_{j\neq i}p_{ij}(1-p_{ij})\big(e_{j}^T \mathbb{E}[A]^{k-1} \bmu_m\big)^2
    \end{align*}
    using $\Xi_2 = \{(1,2)\}$ and $\tuples{[n]\setminus\{i\}}{1} = [n]\setminus\{i\}$. 
    Recall that $y_i =\ell$. Then for any $j$ with $y_j \in \mathcal{I}_\ell$, by assumptions~\ref{as:scale}--\ref{as:nu:upper}, $p_{ij}\geq c_B \nu_n$ and  %
    $1-\pmax \ge c_\nu$, hence $p_{ij}(1-p_{ij}) \ge c_B c_\nu \,\nu_n / n$ for $j \neq i: j \in \mathcal I_\ell$.
    Letting $c_1 = c_B c_\nu$, we have
    \[
    \sum_{m}  
    \Thit_{im}(2)
    \ge c_1 \frac{\nu_n}{n} \sum_{\substack{j \neq i:\\ y_j\in\mathcal{I}_{\ell}}}  \sum_m \big(\bmu_m^T \mathbb{E}[A]^{k-1} e_j \big)^2,
    \]
    using the symmetry of $A$ and invariance of a scalar to transpose. 
    The inner sum over $m$ is equal to $\norm{ \bmu \,\ex[A]^{k-1} e_j}_2^2$ since $\bmu_m^T$ is the $m$th row of $\bmu$.
    Next, we note that 
    \begin{align*}
        \norm{ \bmu \,\ex[A]^{k-1} e_i}_2
        \;\le\; 
        \opnorm{\bmu}\cdot
        \opnorm{\ex[A]}^{k-1}   
        \;\le\; \sqrt{d} \, C_\mu \nu_n^{k-1}.
    \end{align*}
    Thus including the term $j = i$ in the sum, we pay only a small price of $\norm{M \ex[A]^{k-1}e_i}_2^2$, which gives
    \[
    \sum_{m}  
     \Thit_{im}(2)
    \;\ge\; c_1 \frac{\nu_n}{n} \Bigl( \sum_{j:\, y_j\in\mathcal{I}_\ell}  \norm{ \bmu \,\ex[A]^{k-1} e_j}_2^2 - d C_\mu^2 \nu_n^{2k-2}\Bigr).
    \]
   Noting that the sum over $j \in [n]$ is the squared Frobenius norm of $\bmu \ex[A]^{k-1}$ %
   finishes the proof.
\end{proof}

Continuing, we have the supplementary lemma which lowerbounds the average Frobenius-norm of $M\ex[A]^{k-1} \Imat_\ell$: 
\begin{lem}\label{lem:M:EA:k-1:lower}
    Assume \ref{as:pi}--\ref{as:sep} and  
    $\min\{n/k,\,\nu_n^\delta /C_B\}\ge 4 C_\mu L / (c_\pi c_\xi)$. 
    Then,
    \[
    \frac{\nu_n}{n}\sum_\ell \pi_\ell \fnorm{M\,\ex[A]^{k-1}\Imat_\ell}^2
    \geq 
    \Bigl( \frac{c_\pi c_\xi^2 d}{8L} \Bigr)
    \nu_n^{2k-1}.
    \]
\end{lem}

Recall the definition of $\kappa_1 = (c_B c_\nu c_\pi c_\xi^2) / (48L)$ in~\eqref{eq:kappa:const}.\begin{prop}\label{prop:Thit:lower}
    Assume~\ref{as:scale}--%
    \ref{as:sep} and growth condition~\eqref{eq:n:growth} on $n$.
    Then, %
    \[
    \frac1n \sum_{i,m}  \Thit_{im}(2) 
    \ge 3 \kappa_1 d \,\nu_n^{2k-1}.
    \]
\end{prop}
\begin{proof}
    
    By Lemma~\ref{lem:Thit:lower}, for all $i \in [n]$ and $\ell \in [L]$,
    \[
    \ind\{y_i = \ell\} \sum_m 
    \Thit_{im}(2) 
    \;\ge\; \ind\{y_i = \ell\} \cdot c_B c_\nu\,  \Bigl(  \frac{\nu_n}{n} \fnorm{ \bmu \,\ex[A]^{k-1} \Imat_\ell}^2 - d C_\mu^2%
    \frac{\nu_n^{2k-1}}n \Bigr).
    \]
    Summing both sides over $i \in [n]$ and $\ell\in [L]$, and dividing by $n$,
    \[
    \frac1n \sum_{i,m}
    \Thit_{im}(2) 
    \;\ge\; c_B c_\nu\,  \Bigl(  \frac{\nu_n}{n} \sum_{\ell} \pi_\ell \fnorm{ \bmu \,\ex[A]^{k-1} \Imat_\ell}^2 - d C_\mu^2%
    \frac{\nu_n^{2k-1}}n \Bigr).
    \]
    Combining Lemmas~\ref{lem:Thit:lower} and~\ref{lem:M:EA:k-1:lower}, we note that if $\min\{(n/k),\, (\nu_n^\delta/C_B)\} \ge 4 C_\mu L / (c_\pi c_\xi)$ and %
    $c_\pi c_\xi^2/ (16 L) \ge C_\mu^2/ n$, then the result follows. These conditions on $n$ can be combined into~\eqref{eq:n:growth}. %
\end{proof}

\subsection{Proof of Theorem~\ref{thm:noise:lower:bound}}
We are now ready to prove the noise lowerbound. 
\begin{lem}\label{lem:Edev2:lower}
    Assume growth conditions~\eqref{eq:n:growth} and~\eqref{eq:nu:growth}, 
    and $r_n \ge 2$.
    We have the following lower bound 
    \[
    \ex(\dev)^2 \ge \kappa_1 d \nu_n^{2k-1}.
    \]
\end{lem}
\begin{proof}
Growth condition~\eqref{eq:nu:growth} can be written as
\begin{align}\label{eq:nu:growth:2}
        2 \maxnorm{\mu}^2 \max\bigl\{(2k-1)^2 n^{-1}%
        , \nu_{n}^{-\epsilon}\bigr\} \le \kappa_1,
\end{align}
We apply the results with $r = 2$, and let $t_* = 2k-1$.
    We have 
    \begin{align*}
        \ex(\dev)^2 \ge \frac1n \sum_{i,m} \ex[\Delta_{im}^2]
        &= \frac1n \sum_{i,m}( \Thi_{im}(2) + \Tlo_{im}(2)) \\
        &\ge \frac1n \sum_{i,m} \Bigl( \Thit_{im}(2) - |\Thi_{im}(2) - \Thit_{im}(2)| - |\Tlo_{im}(2)|\Bigr)  
    \end{align*}
    By Proposition~\ref{prop:Thit:lower},
    $%
      \frac1n \sum_{i,m}  \Thit_{im}(2)
    \ge 3 \kappa_1 d  \nu_n^{t_*}$. From Lemma~\ref{lem:Thi_proxy_dev}, 
    \[
    \sum_m |\Thi_{im}(2) - \Thit_{im}(2)| \le 2 d t_*^2 \pmax \maxnorm{\mu}^2 \nu_n^{t_*-1}
    \]
    and from Lemma~\ref{lem:T_grwth}, for $2 \le r_n$, we have 
    \[
     \sum_m |\Tlo_{im}(r)| \le  d (\sqrt{2} \,\maxnorm{\mu} )^{2} \cdot \nu_n^{t_*-\epsilon}.
    \]
    By assumption, $2 t_*^2 \pmax \maxnorm{\mu} ^{2} \nu_n^{-1} \le \kappa_1 $ and $2  \maxnorm{\mu} ^{2} \nu_n^{-\epsilon} \le \kappa_1$. Combining, we obtain
    $
    \ex(\dev)^2 \ge \kappa_1 d \nu_n^{t_*}. 
    $ which is the first assertion.
\end{proof}

Recall the definition of $\kappa_2 =  8 \big( 32 \maxnorm{\mu}^4 + (8C_1 \sigma)^4\big)$ in~\eqref{eq:kappa:const}.
\begin{lem}\label{lem:2:4:mom:ratio}
    Under the assumptions of Lemma~\ref{lem:Edev2:lower}, further assume $r_n \ge 4$. Then,
    \begin{align*}
        \frac{\bigl(\ex \dev^2\big)^2}{\ex\, \dev^4} \ge \frac{\kappa_1^2}{\kappa_2}.
    \end{align*}
\end{lem}
\begin{proof}
    We have the upper bound
    \begin{align*}
        \ex(\dev)^4 &\;\leq \;\frac{d}{n} 2^3 \sum_{i,m}\bigl( \ex (\Delta_{im}^4) + \ex(\Delta_{im}^\eps)^4 \bigr) \\ &\;\le\; 8 d^2 \max_{i,m} \bigl( \ex (\Delta_{im}^4) + \ex(\Delta_{im}^\eps)^4 \bigr) 
    \end{align*}
    By Lemma~\ref{lem:T_grwth}, we have for $4 \le r_n$,
    \begin{align*}
        \ex (\Delta_{im}^4) \;\le\;
        |\Thi_{im}(4)| + |\Tlo_{im}(4)|  
        \;\le\; 2 |\Thi_{im}(4)| \;\le\; 2 (\sqrt{4} \maxnorm{\mu})^4 \nu_n^{4k-2}
    \end{align*}
    and from~\eqref{eq:delt_eps_rgwth},
    \[
     \ex(\Delta_{im}^\varepsilon)^4 \leq \bigl(%
     4
     C_1
    \sigma \nu_n^{k-1/2} \sqrt{4} \bigr)^4  = (%
    8
    C_1 \sigma)^4 \nu_n^{4k-2}.
    \]
    With $\kappa_2$ as defined above, we have 
    $\ex(\dev)^4 \le \kappa_2 d^2 \nu_n^{4k-2}$. Combining with the lower bound in Lemma~\ref{lem:Edev2:lower}, we  get the result.
\end{proof}

\begin{proof}[Proof of Theorem~\ref{thm:noise:lower:bound}]
Applying the Paley--Zygmund inequality to non-negative quantity $\dev^2$ yields
     \[
    \mathbb{P}\big(\dev^2 \geq \eta\, \ex\dev^2\big) \geq (1-\eta)^2  
    \frac{\bigl(\ex\dev^2\bigr)^2}{\ex\, \dev^4}.
    \]
    Using Lemma~\ref{lem:Edev2:lower} on the LHS and Lemma~\ref{lem:2:4:mom:ratio} on the RHS, we have the result.
\end{proof}

\section{Conclusion}
\label{sec:conclude}
In this work, we provide sharp bounds on the signal-to-noise ratio for graph-aggregated features. We show that these features have a fundamental information rate which is invariant to the overall depth of the network $k$. These are features which underpin many GNN architectures and are the differentiator between GNNs and traditional feed-forward neural networks. As such, the knowledge that a feature information limit exists and is attainable for all networks with depth $k\geq 1$, %
is likely to influence
future GNN architecture choices for empirical studies.

Our results are the first of their kind with respect to the generality of their assumptions. We work with the common CSBM %
but make no assumptions on the connectivity structure and allow for potentially disconnected clusters. Furthermore, our results bring to light how much the signal and noise %
are
intertwined for GNNs, that a separation in the features necessarily comes with %
a %
scaling of the noise. Additionally, the upperbound portions of our %
noise bounds hold for the general class of inhomogeneous Erd\"{o}s-R\'{e}nyi models. This is a family of generative graph models which supersede CSBMs and include models like the random dot-product graph (RDPG).

Our technical contributions provide a mix of results from matrix analysis and random matrix theory. Results related to walks and their combinatorics, provide simple recipes to extend our %
anylysis to other generative frameworks. Other results bring to light the interplay between path-counting and edge probabilities, where certain subgraphs (mainly non-tree subgraphs) are less likely to occur given that they contain fewer %
vertices than %
unique edges. As we have shown, it is exactly this interplay which leads to the presence of dominant walk types %
in the noise contribution. Lastly, the presence of %
dominant walk types %
allows for a clean analysis since, to the leading order, this means the noise can be approximated by special polynomial forms (refer to Section~\ref{sec:proxy:for:Thi} for more details).

\bibliography{refs}

\appendix

\section{Proofs For Signal Argument}\label{sec:proofs:for:signal}

We start with a lemma matrix monomial deviations.
\begin{lem}\label{lem:poly_dev}
    Let $U,V \in\mathbb{R}^{n\times n}$ then for any $k\in\mathbb{N}$,
    \begin{equation}\label{eq:UV_dev}
        \norm{U^k - V^k} \leq k2^{k-2} \norm{U-V}\,(\norm{U-V}^{k-1}+ \norm{V}^{k-1}).
    \end{equation}
    Alternatively one can also derive
    \begin{equation}\label{eq:UV_max}
        \norm{U^k - V^k} \leq k \norm{U-V}\,(\max\{\norm{U}, \norm{V})^{k-1}.
    \end{equation}
\end{lem}
\begin{proof}
    Consider the matrix valued function $f(U)= U^k$. By \cite[Theorem X.4.5]{Bhatia97} one has, for matrix inputs $U$ and $V$,
    \[
    \norm{f(U)- f(V)} \leq \norm{U-V} \sup_{W\in \mathscr{L}(U,V)} \norm{\partial f_W}
    \]
    where $\mathscr{L}(U,V) = \{\eta U+(1-\eta) V:\,\eta\in[0,1]\}$ is the line segment between matrices $U$ and $V$,  $\partial f_W$ is the Fr\'{e}chet derivative of $f$ at $W$ and %
    $\norm{\partial f_W}$ is the operator norm of the derivative defined %
    as
    \[
    \norm{\partial f_W} \coloneqq \sup_{\norm{H}=1} \norm{\partial f_W(H)}.
    \]
    The Fr\'{e}chet derivative of a matrix monomial can be straightforwardly computed as
    \[
    \partial f_W(H) = \sum_{\ell=0}^{k-1} W^\ell H W^{k-1-\ell}.
    \]
    By submultiplicativity of operator norm
    \begin{align*}
        \sup_{\norm{H}=1} \norm{\partial f_W(H)} \;\leq\; \sum_{\ell=0}^{k-1} \sup_{\norm{H}=1} \norm{W^\ell H W^{k-1-\ell}} 
        \;\leq\; k\norm{W}^{k-1}.
    \end{align*}
    Next with $W = \eta U + (1-\eta) V$
    \begin{align*}
        \norm{W}^{k-1} &= \norm{\eta (U-V) + V}^{k-1}\\
        &\leq (\norm{U-V} + \norm{V})^{k-1},
    \end{align*}
    where we recall $\eta\in[0,1]$. By convexity of $x\mapsto x^{k-1}$ on $[0,\infty)$,
    \[
    (\norm{U-V} + \norm{V})^{k-1} \leq 2^{k-2}(\norm{U-V}^{k-1} + \norm{V}^{k-1})
    \]
    Stringing along the previous inequalities 
    proves~\eqref{eq:UV_dev}.
    To obtain~\eqref{eq:UV_max} simply bound $\norm{W}$ as
    \begin{align*}
        \norm{W}^{k-1} &= \norm{\eta U + (1-\eta)V}^{k-1}\\
        &\leq (\eta \norm{U} + (1-\eta)\norm{V})^{k-1}\\
        &\leq \eta \norm{U}^{k-1} + (1-\eta)\norm{V}^{k-1}\\
        &\leq (\max\{\norm{U},\,\norm{V}\})^{k-1}.
    \end{align*}
    using the convexity of $x \mapsto x^{k-1}$ in the third line.
\end{proof}

The final piece needed to prove Theorem~\ref{thm:signal:main} is the following concentration result for exponentiated sub-Gaussian and bounded random matrices.

\subsection{Proof of Lemma~\ref{lem:Ak:concentration}}
     Recall that wlog $p_{\rm max} = \nu_n /n$. So $\norm{\ex[A]}\leq np_{\rm max}=\nu_n$. Apply~\eqref{eq:UV_dev} of Lemma~\ref{lem:poly_dev} with $U = A$ and $V = \ex[A]$ to obtain
    \begin{align}\label{eq:bandeira}
    \norm{A^k-\ex[A]^k} \le C_{1,k} \Bigl(  \norm{\Delta}^k + \nu_n^{k-1}\norm{\Delta} \Bigr).
    \end{align}
    where $C_{1,k} = k 2^{k-2}$, and $\Delta := A - \ex[A]$.

    Next consider a spectral concentration result from~\cite[Corollary 3.12; Remark 3.13]{Bandeira16} where, for a random matrix $U$ with independent sub-Gaussian entries, one has
    \[
    \mathbb{P}\big(\norm{U} > C \widetilde{\sigma} + t\big) \leq n e^{-t^2 /c\widetilde{\sigma}_*^2}
    \]
    for universal constants $C > 1$ and $c > 0$, with $\widetilde{\sigma} \coloneqq \max_i \sqrt{\sum_j \ex[U_{ij}^2]}$ and $\widetilde{\sigma}_*\coloneqq \max_{ij} \norm{U_{ij}}_\infty$. For $U= A-\ex[A]$ these parameters become $\widetilde{\sigma} \leq \sqrt{\nu_n}$ and $\widetilde{\sigma}_*\leq 1$. Consider the event
    \[
    \mathcal{E} = \bigl\{\norm{\Delta}\leq C_{2,k} \sqrt{\nu_n} \bigr\}
    \]
    where $C_{2,k} = C+\sqrt{(c / c_\nu') (k+1)}$. Then, taking $t = (C_{2,k} - C) \sqrt{\nu_n}$ in~\eqref{eq:bandeira}, we have 
    \[
    \pr(\mathcal E^c) \le n e^{-(k+1) \nu_n / c_\nu'} \le n e^{-(k+1) \log n} \le n^{-k}.
    \]
    Then, on the one hand,
    \[
    \ex[ \norm{\Delta}^k \ind_{\mathcal E^c}] \le n^k \pr(\mathcal E^c)  \le 1 \le \nu_n^{k-1/2}.
    \]
    On the other hand,
    \[
    \ex[ \norm{\Delta}^k \ind_{\mathcal E}] \le C_{2,k}^k \nu_n^{k/2} \le C_{2,k}^k \nu_{n}^{k-1/2}.
    \]
    Putting the pieces together we have
    \[
    \ex \norm{A^k - \ex[A]^k} \le C_{1,k} \Bigl( (1+C_{2,k}^k) \nu_n^{k-1/2} + \nu_n^{k-1} (1+C_{2,1}) \nu_n^{1/2}
    \Bigr)
    \]
    proving the result with constant, e.g., $C_k =4 C_{1,k} C_{2,k}^k$ using $C_{2,k} \ge 1$ and $C_{2,k} \ge C_{2,1}$.

\section{Concentration from Moments}
\label{app:moment:concent}

We have the following concentration inequality from moments:
\begin{lem}[Sub-Weibull concentration]\label{lem:moment:concent1}
Let $\eta > 0$.  Assume that for all $r \in 2\nats$, we have
    \begin{align}\label{eq:general:moment:growth}
        \ex[|\Delta|^r] \le \Bigl( K (C \eta r)^\eta\Bigr)^{r}.
    \end{align}
    Then, 
    \[
    \pr\Bigl(|\Delta| \ge K x^\eta\Bigr) \le \exp\Bigl( - \frac{x}{2Ce}\Bigr)\qquad\text{for}\; x \geq 4\eta C e.
    \]
\end{lem}

Lemma~\ref{lem:moment:concent1} follows from a more general statement for partial moment control:
\begin{lem}\label{lem:moment:concent2}
    Let $\eta > 0$ and $r_0 \in 2\nats \cup \{\infty\}$. Assume that for all even integers $r\leq r_0$, we have
    \begin{equation}\label{eq:partial_moms}
        \ex|\Delta|^r\leq  \Big(K(C\eta r)^\eta\Big)^r.
    \end{equation}
    Then~\eqref{eq:partial_moms} holds for all real $r \in [2,r_0]$ with $C$ replaced with $2C$. Moreover, if $x \geq 4C\eta e $
    \[
    \mathbb{P}\Big(|\Delta| \geq K x^\eta\Big) \leq \begin{cases}
    \exp\big(-\frac{x}{2Ce}\big), &\text{for}\;x\leq 2C\eta er_0\\
    \big(\frac{2\eta Cr_0}{x}\big)^{\eta r_0},&\text{for}\;x > 2C\eta e r_0
    \end{cases}\leq \exp\Bigl(-\min\Big\{\frac{x}{2Ce},\,\eta r_0\Big\}\Bigr).
    \]
\end{lem}
\begin{proof}
    Let us first establish the claim that~\eqref{eq:partial_moms} holds for all real $r \in [2,r_0]$ with $C$ replaced with $2C$. If $r_0 = 2$ this is trivial. Otherwise for $r_0 \geq 4$, we will use the log-convexity of $L_p$ norms, $(\ex|\Delta|^p)^{1/p}$. Fix $r \leq r_0 - 2$ and $\theta \in (0,1)$. Log-convexity implies
    \[
    (\ex|\Delta|^p)^{1/p} \leq (\ex|\Delta|^r)^{(1-\theta)/r}(\ex|\Delta|^{r+2})^{\theta/(r+2)}
    \]
    where $p$ is given by $\frac{1}{p} = \frac{1-\theta}{r} + \frac{\theta}{r+2}$. Applying~\eqref{eq:partial_moms} with both $r$ and $r+2$, we obtain
    \begin{align*}
        (\ex|\Delta|^p)^{1/p} &\leq [K(Cr\eta)^\eta]^{1-\theta}[K(C(r+2)\eta)^\eta]^\theta\\
        &= K(Cr^{1-\theta}(r+2)^\theta \eta)^\eta.
    \end{align*}
    We have $r^{1-\theta} (r+2)^\theta = 2 (r/2)^{1-\theta}(r/2+1)^\theta\leq 2(r/2)^{1-\theta}r^\theta \leq 2^\theta r \leq 2p$, since $r < p$. This gives $(\ex|\Delta|^p)^{1/p} \leq K(2Cp\eta)^\eta$ which is the desired bound.

    Next, consider the tail bound. Let us redefine $2C$ to be $C$ for simplicity. By Markov inequality, for all real $r\geq 2$,
    \[
    \mathbb{P}(|\Delta|\geq u)\leq \frac{\ex|\Delta|^r}{u^r}\leq \Big(\frac{K(Cr\eta)^\eta}{u}\Big)^r.
    \]
    Using a change of variable $u = K x^\eta$, we have
    \[
    \mathbb{P}\Big(|\Delta|\geq Kx^\eta\Big)\leq \Big(\frac{(C\eta r)^\eta}{x^\eta}\Big) = \alpha^r r^{\eta r} = \exp(f(r))
    \]
    where $\alpha = (C\eta/x)^\eta$ and $f(r) = r\log \alpha + \eta r\log r$. Let us now optimize over $r$. The function $f$ has first and second derivative $f'(r) = \log \alpha + \eta \log r+ \eta$ and $f''(r) = \eta/r > 0$. Hence, $f$ is strictly convex on $(0,\infty)$ and achieves its global minimum at $r^* = e^{-1-(\log \alpha)/\eta} = (1/e)\alpha^{-1/\eta}$. If $r_0 \geq r^* \geq 2$, that is $x\leq C\eta e r_0$, then we achieve the first case tail bound
    \[
    \mathbb{P}\Big(|\Delta|\geq K x^\eta\Big)\leq \exp\Big(-\frac{x}{2Ce}\Big).
    \]
    Otherwise if $x > C\eta e r_0$, one has $r_0 < r^*$. In this case we use the next best value of $r$ for our probability upperbound $\exp(f(r))$, that is $r = r_0$. Altogether
    \[
    \mathbb{P}\Big(|\Delta| \geq K x^\eta\Big) \leq \begin{cases}
    \exp\big(-\frac{x}{Ce}\big), &\text{for}\;x\leq C\eta er_0\\
    \big(\frac{\eta Cr_0}{x}\big)^{\eta r_0},&\text{for}\;x > C\eta e r_0.
    \end{cases}
    \]
    We note that the two bounds match at the boundary. We can summarize as
    \begin{align*}
        \mathbb{P}\Big(|\Delta|\geq K x^\eta \Big)&\leq \max\Big\{e^{-x/(Ce)},\Big(\frac{Cr_0\eta}{x}\Big)^{\eta r_0}\Big\}\\
        &\leq \exp\big(-\min\{x/(Ce),\eta r_0\}\big)
    \end{align*}
    where we evaluated the second bound at $x=C\eta er_0$. Replacing $C$ with $2C$ finishes the proof.
\end{proof}

\section{Counting Lemmas}
\label{app:counting:lemmas}

\subsection{Proof of Lemma~\ref{lem:N_tij_walks}}
	Let $i,j \in [n]$ with $i \neq j$.
    We will over-enumerate the walks in $\mathcal{N}_t(i,j)$, each of which contain, at most, $t+1$ unique vertices.
    We construct a potential vertex set for a walk $w$ in $\mathcal{N}_t(i,j)$ by first fixing $\{i,j\}$ and selecting the remaining $t-1$ vertices from $[n]\setminus\{i,j\}$. There are exactly $\binom{n-2}{t-1}$ ways to do this.
    Next, starting from node $i$, each outgoing vertex forms an additional edge in our walk. As there can be no self-loops and the last outgoing vertex $j$ is already determined, there are, at most, $t^{k-1}$ ways to select such edges. Altogether, there are, at most, $\binom{n-2}{t-1}t^{k-1}$ ways to construct a walk in $\mathcal{N}_t(i,j)$, which is the desired result.

 \subsection{Proof of Lemma~\ref{lem:Nii}}

     The proof for the cardinality of $\bm\mathring{\mathcal{N}}_t(i,i)$ follows similarly to the case $i \neq j$ (Lemma~\ref{lem:N_tij_walks}). The key realization is that any walk $w \in \bm\mathring{\mathcal{N}}_t(i,i)$ can have at most $t$ unique vertices. This is by definition since the undirected graph $G(w)= (V,E)$ is, one, connected ($|V|\leq |E| +1$) and, two, not a tree ($|V|\neq |E| + 1$). So walk $w\in \bm\mathring{\mathcal{N}}_t(i,i)$ must have at most $t$ unique vertices.

\medskip
The rest of the proof is devoted to bounding $|\bm{\breve}{\mathcal{N}}_t(i,i)|$.
 For $w\in\bm{\breve}{\mathcal{N}}_t(i,i)$, the edges of $G(w)$ must all be traversed an even number of times.
 We refer to a traversal that occurs on an odd (resp. even) visits of an edge %
 as positive (resp. negative) traversals. Given the backtracking nature of walks $w\in\bm{\breve}{\mathcal{N}}_t(i,i)$, it is sufficient to keep track of positive edge traversals to reconstruct $w$, and so a recipe for bounding $|\bm{\breve}{\mathcal{N}}_t(i,i)|$ is to list all viable $G(w)$ and count the possible positive edge traversals.

Recall that $G(w)= (V,E)$ is loop-less and connected, hence a tree. The vertices $V$ have a natural ordering on $\mathbb{N}$ that allows for an encoding of $G(w)$ using the so-called ``Balanced parentheses sequence'' or ``Dyck word'' representation of the tree.
The unlabeled graphs of these encodings, which we will call $\mathcal{B}(w)$, are precisely the (rooted) plane trees with $t+1$ vertices and their count is exactly the Catalan number:
\[
\mathfrak B \coloneqq \{\mathcal{B}(w):\, w\in\bm{\breve}{\mathcal{N}}_t(i,i)\}, \quad 
|\mathfrak B| =C_t.
\]
Next, since $\mathcal{B}(w)$ has the same number of edges as $G(w)$, we can construct edge selections on $\mathcal{B}(w)$ that correspond to edge %
traversals on $G(w)$, modulo a choice of vertices $V$. Selecting for the vertices $V$, where the root $i$ is determined, gives an additional factor of $\binom{n-1}{t}$.

Lastly, we must perform an edge selection for each $\mathcal{B}(w)$, in particular we must select for $k/2$ edges which will later be our positive traversals. An arbitrary selection of edges gives a factor of $D_{k,t} = t! \stirling{k/2}{t}$ where $ \stirling{m}{t}$ is the Stirling number of second kind. To see this note that we arranging $t$ objects in $m=k/2$ slots where (a) each object must be used at least once (to cover the tree), (b) objects can be used multiple times and (c) the order matters. This is equivalent to the number of surjections from a set of $m$ elements to a set of $t$ elements which is given by $ t!  \stirling{m}{t}$.

Certain edge selections on $\mathcal{B}(w)$ will translate to disconnected (i.e., invalid) walks on $G(w)$, yielding an overcount in the number of walks $w\in\bm{\breve}{\mathcal{N}}_t(i,i)$ with undirected graph $G(w)$. 
An exception is when $G(w)$ is a star.
Here, any edge selection on $\mathcal{B}(w)$, with any vertex set $V$, gives a valid sequence of positive edge traversals for the star-shaped $G(w)$ with vertices $V$. In other words, the two star graphs in $\mathfrak B$ (corresponding to whether the root is a hub or a spoke) are maximal in terms of the counts of valid walks they produce, hence the count in this case provides an upper bound on the counts for all elements of $\mathfrak B$.

Combining all factors (encoding trees, vertex choice, %
maximality of stars) gives the desired bound
\[
 |\bm{\breve}{\mathcal{N}}_t(i,i)| \le C_t \cdot \binom{n-1}{t}\cdot D_{k,t}.
\]
For~\eqref{eq:loopless:crude:bound}, we note that $C_t \le \frac1t (2e)^t$, and $t! \cdot S(k/2,t) \le t^{k/2}$ which says that the number of surjections from $[k/2]$ to $[t]$ is at most the total number of functions. Moreover, $\binom{n-1}{t} \le (en / t)^t$. Combining, we get the second claimed upper bound, concluding the proof. %

\subsection{Proof of Lemma~\ref{lem:counting}}

    Consider collapsing the walk sequence $\bm{w}\in\mathcal{W}^r$  to %
    a single walk $\widetilde{w}$. That is, for walk sequence $\bm{w} = (((i_\ell^s,j_\ell^s), \ell\in[k]), s\in[r])$ let $F(\cdot)$ be the in-place tuple flattening such that 
	\[
	\widetilde{w} \coloneqq F(\bm w) = ((i_\ell^s,j_\ell^s), \ell\in[k],s\in[r]).
	\]
	The tuple flattening function $F$ is a reversible process for known $r$ and $k$, meaning there is an isomorphism defined by $F$ between $\widetilde{\mathcal{N}}_{r,t,v} \coloneqq \{F(\bm w):\,\bm{w} \in \mathcal{N}_{r,t,v}\}$ and $\mathcal{N}_{r,t,v}$. As such any cardinality upperbound on $\widetilde{\mathcal{N}}_{r,t,v}$ will hold for $\mathcal{N}_{r,t,v}$.\newline
	One way to count the walks of $ \widetilde{\mathcal{N}}_{r,t,v}$ is to first choose $v-1$ non-root nodes (for which there are $\binom{n-1}{v-1}$ options). After which rename the non-root nodes without loss of generality to $1,\ldots,(v-1)$ and the rename the root node $i$ to $n$. The walk $\widetilde{w}$ takes $rk$ steps on $\{1,\ldots,v-1\}\cup\{n\}$. Since there are, at most, $v-1$ options to choose from at each step, we may over-enumerate the walks $\widetilde{w}\in\widetilde{\mathcal{N}}_{r,t,v}$ and arrive at the simple upperbound
	\[
	|\mathcal{N}_{r,t,v}| = |\widetilde{\mathcal{N}}_{r,t,v}| \leq (v-1)^{rk}\binom{n-1}{v-1}.
	\]
    For the second assertion, we have
\begin{align*}
    \sum_{v = 1}^{b+1}|\Nc_{r,t,v}| \le 
\sum_{v = 1}^{b+1} (v-1)^{rk}\binom{n-1}{v-1} = 
\sum_{v=0}^{b} v^{rk} \binom{n-1}{v} \le b^{rk} \sum_{v=0}^{b}\binom{n}{v} \le b^{rk} \Bigl( \frac{e n}{b} \Bigr)^b. %
\end{align*}
The proof is complete.

\section{Proofs for Noise Upperbound}

\subsection{Proof of Lemma~\ref{lem:rho:control}}
We have
\begin{align}
    \Bigl| \mathbb{E} \prod_{s=1}^r(A_{\bm{w}^s} - \mathbb{E}[A_{\bm{w}^s}]) \Bigr| &\leq \mathbb{E}\Big[\prod_{s=1}^r(|A_{\bm{w}^s}| + \mathbb{E}|A_{\bm{w}^s}|)\Big]\nonumber\\
    &=  \ex \Bigl[ \sum_{S \subset [r]} \prod_{s \in S }  |A_{\wb^s}| \cdot \prod_{u \in S^c} \ex| A_{\wb^{u}}| \Bigr] \nonumber \\
    &= \sum_{S \subset [r]} \Bigl[ \ex[ A_{\wb^{S}}] \cdot \prod_{u \in S^c}  \ex| A_{\wb^{u}}|\Bigr] 
    \le 2^r \cdot \ex[A_{\wb}]
    \label{eq:Aw_upper} 
\end{align}
where the final inequality is by Lemma~\ref{lem:A:moment:inequality}. Then, we have $\ex[A_{\wb}] \le \pmax^{|[w]|}$ by~\eqref{eq:temp:48}.

\smallskip
To control $\varrho_2(\wb)$, we note that for $J = \{\proj(\wb^s): s \in [r]\} \subset [n]$ and some integers $a_j \ge 1$ with $\sum_j a_j = r$, we have 
\begin{align*}
\Bigl| \ex \prod_{s=1}^r (x_{\proj(\bm{w}^s)})_m \Bigr| \le \ex \Bigl(\prod_{s=1}^r |(x_{\proj(\bm{w}^s)})_m|\Bigr)  
= \ex \Bigl( \prod_{j \in J} |(x_{j})_m|^{a_j} \Bigr) 
\le \max_{j \in J} \ex|(x_j)_m|^r
\end{align*}
where the equality is by independence of $x_j, j\in J$, and  the last step follows from Lemma~\ref{lem:mixed:log:convexity}. Combining with
\begin{align*}
  \ex|(x_j)_m|^r 
  &\le 2^{r-1} \big(\ex |\eps_{im}|^r + |(\mu_{y_i})_m|^r)
  \\
  &\le 2^{r-1} \big((C_1\sigma r^{1/2})^r + %
  \infnorm{\mu_{m*}}^r\big),
\end{align*}
where the last line used Lemma~\ref{lem:subgauss_moms},
the proof is complete.

\subsection{Proof of Lemma~\ref{lem:equiv_partition}}
		With $\wb$ overlapping we have for all $q$, $|\Gamma_q| \ge 2$. Hence, $2Q \le \sum_{q=1}^Q |\Gamma_q| = r$ and the upper bound on $Q$ follows. Next fix $q$,
		\begin{align*}
			|[\wb]^{\Gamma_q}| &\le |[\wb^1]| + \sum_{s=2}^{|\Gamma_q|} | [\wb^s] \setminus [\wb^1]| \\
			&\le k + (|\Gamma_q|-1)(k-1) = |\Gamma_q|(k-1) + 1.
		\end{align*}
		Then combining across partitions $\{\Gamma_q\}_q$,
		\begin{align*}
			|[\wb]| = \sum_{q=1}^Q |[\wb]^{\Gamma_q}| &\le \Bigl(\sum_{q=1}^Q |\Gamma_q|\Bigr) (k-1) + Q \\
			&= r(k-1) + Q \le rk - \lceil r/2\rceil
		\end{align*}
		using $Q \leq \lfloor r/2\rfloor$.

\subsection{Proof of Lemma~\ref{lem:T_grwth}}
Throughout, we fix $i, m$ and $r$.
We start with the bound for $\Tlo$. To simplify the notation, let
\[
\beta_r :=  (\infnorm{\mu_{m*}} \sqrt{r})^{r}.
\]
From Lemma~\ref{lem:rho:control}, for any $\wb \in \Nc_{r,t,v}$, we have 
\begin{align}
|\rho(\wb)| &= |\varrho_1(\bm w)| \cdot |\varrho_2(\bm w)| \notag \\
&\le (4\max\{C_1 \sigma r^{1/2}, \infnorm{\mu_{m*}}\})^r\, \pmax^t \notag \\
&\le %
(\const_{0,m} \infnorm{\mu_{m*}} r^{1/2})^{r} \pmax^t \le \beta_r\, \const_0^r\,  \pmax^t \label{eq:rho:w:bound}
\end{align}
where $\const_{0,m}  \infnorm{\mu_{m*}} = 4\max\{C_1 \sigma, \infnorm{\mu_{m*}}\}$ and $\const_0 = \max_m \const_{0,m}$ by definition. 

Let us first bound $\Tlo$. We have 
\[
\Tlo = \sum_{t=1}^{t_*}\sum_{v=2}^{b_t+1} \sum_{\bm{w} \in\mathcal{N}_{r,t,v}} \varrho(\wb)
\]
where $b_t = t \wedge (t_*-1)$. That is, the inner sum goes from $v=2$ to $v=t+1$ unless $t = t^*$ in which goes it only goes to $v = t^*$. %
Using~\eqref{eq:rho:w:bound},
\[
\frac{|\Tlo|}{\beta_r} \le %
\const_0^r \sum_{t=1}^{t_*}  \pmax^t \sum_{v=2}^{b_t+1} |\Nc_{r,t,v}|.
\]
By Lemma~\ref{lem:counting}, we have
$
 \sum_{v=2}^{b_t+1} |\Nc_{r,t,v}| \le b_t^{r/2} b_t^{t_* - b_t} (en)^{b_t}
$
since $t_* = rk - r/2$.
Using this bound and $\pmax = \nu_n / n$, we have
\[
\frac{|\Tlo|}{\beta_r} \le  \const_0^r  \, t_*^{r/2} \sum_{t=1}^{t_*}  \nu_n^t n^{-t} b_t^{t_* - b_t} (en)^{b_t}.
\]
Multiplying and dividing by $\nu_n^{t_*}$ and rearranging, then separating the term $t = t_*$ from $t < t_*$, we have 
\begin{align*}
\frac{|\Tlo|}{ \beta_r\, \nu_n^{t_*}} 
	&\;\le\; \const_0^r\, t_*^{r/2} \sum_{t=1}^{t_*}  \nu_n^{t-t_*} n^{b_t-t} b_t^{t_* - b_t} e^{b_t}, \\
	&\;=\; \const_0^r\,t_*^{r/2} \Bigl[ n^{-1}(t_*-1) e^{t^*-1} + \sum_{t=1}^{t_*-1}  \nu_n^{t-t_*} t^{t_* - t} e^{t} \Bigr] \\
	&\;\le\; \const_0^r\,t_*^{r/2} e^{t^*-1} \Bigl[n^{-1}(t_*-1)  +  \sum_{t=1}^{t_*-1}  \Bigr(\frac{t}{\nu_n}\Bigl)^{t_*-t}   \Bigr].
\end{align*}
By assumption $\kappa_0^r\, t_*^{r} e^{t_*} \le \frac13 \nu_n^{1-\epsilon}$, which implies 
$\kappa_0^r\, t_*^{r/2+1} e^{t_*} \le \frac13 \nu_n^{1-\epsilon}$. (The result holds under this slightly weaker form). Let $\rho = \kappa_0^{-r} t_*^{-r/2} e^{-t_*} \nu_n^{-\epsilon}$. Then, $t_* \le \frac13 \rho \nu_n$,
hence
\begin{align*}
    \sum_{t=1}^{t_*-1}  \Bigr(\frac{t}{\nu_n}\Bigl)^{t_*-t}   \le 
    \sum_{t=1}^{t_*-1}  \rho^{t_*-t}  \le \sum_{u=1}^\infty \rho^u \le 2 \rho.
\end{align*}
using $\rho \le 1/2$ since $\nu_n \ge 1$ and $t_* \ge 1$. We obtain
\begin{align*}
\frac{|\Tlo|}{ \beta_r\, \nu_n^{t_*}} 
	&\le 
 \const_0^r \, t_*^{r/2}e^{t_*}\Bigl[ 
 t_* n^{-1}  + \frac23  \const_0^{-r} \, t_*^{-r/2} e^{-t_*} \nu_n^{-\epsilon}\Bigr] \le \nu_n^{-\epsilon}
 \end{align*}
 where we have used $\const_0^r \, t_*^{r/2+1} e^{t_*} n^{-1} \le\frac13 \nu_n^{1-\epsilon} n^{-1} \le \frac13 \nu_n^{-\epsilon}$. This proves the claim upper bound on $\Tlo.$

\medskip
\newcommand\Gstar{G^*}
Next, we prove the bound for $\Thi$. 
For all $\wb \in \Nc_*$, the unlabeled graph associated with the tree $G(\wb)$ is the same, namely the $\Gstar$ graph described in Section~\ref{sec:char_Nrt} and depicted in Figure~\ref{fig:G:star} for a given example. In other words, as $\wb$ ranges over $\Nc_*$, $G(\wb)$ ranges over all possible labelings of the vertices of $\Gstar$ with $t^*$ distinct elements from $[n] \setminus \{i\}$. There are 
    $|\tuples{[n]\setminus \{i\}}{t_*}|$ such labelings, hence
    \begin{align*}
        |\Nc_*|
        \;=\;
        (r-1)!! \, |\tuples{[n]\setminus \{i\}}{t_*}|
        \;\leq\; (r-1)!!\, n^{t_*}.
    \end{align*}
    By Lemma~\ref{lem:max:rho:star}, for every $\bm{w}\in\mathcal{N}_{r,t_*,t_*+1}$
    \[
    |\rho(\bm{w})| \leq  \infnorm{\mu_{m*}}^r \cdot p_{\rm max}^{t_*}.
    \]
    Altogether, $|\Thi| \le (r-1)!!\, \infnorm{\mu_{m*}}^r (n \pmax)^{t_*}$, which is the first claim. The double factorial is connected to the Gamma function for odd valued inputs. In particular,
    $
    (r-1)!! = \frac{2^{r/2}}{\sqrt{\pi}} (r/2 - 1/2)! \leq r^{r/2}.
    $

\section{Proofs for the Noise Lower bound}

\subsection{Proof of Lemma~\ref{lem:tree:character}}
  For (a), each $\Gamma \in \Xi_r$ is generated by some overlapping $\bm{w} \in \mathcal{N}_{r,t_*,t_* +1}$. The overlapping nature of each walk in the walk sequence $\bm{w}$ enforces that $|\Gamma_q| \geq 2$ for all $q$. The number of equivalences classes, $Q$, is maximized when each $|\Gamma_q|$ is minimized since, by construction, $\Gamma = \{\Gamma_q\}_{q=1}^Q$ must satisfy $\bigsqcup_{q\in[Q]}\Gamma_q = [r]$. In the case $r\in 2\mathbb{N}$, $Q = r/2$ so $\{\Gamma_q\}_{q=1}^Q$ must contain pairs of elements from $[r]$ such that $\Gamma_q \cap \Gamma_{q'} = \varnothing$ for $q\neq q'$ and $|\Gamma_q| =2$ for $q\in [Q]$. This is a perfect matching on the vertex set $[r]$.

  For (b), we note that since each $\wb^s$ is a walk and the walks have common first element $i$, the graph $G(\wb)$ is connected. That is, $G(\wb)$ is a connected graph with  $t_*+1$ nodes and $t_*$ edges, hence a tree. We designate $i$ as its root.   

  For (c), the partition $\Gamma(\bm{w})$ guarantees that subgraphs $G(\bm{w}^{\Gamma_q})$ and $G(\bm{w}^{\Gamma_{q'}})$ cannot share an edge for $q,q'$ distinct. Furthermore by (b), if $G(\bm{w}^{\Gamma_q})$ and $G(\bm{w}^{\Gamma_{q'}})$ share a vertex other than the root $i$, then an undirected cycle forms, contradicting the tree property of $G(\bm{w})$.

  For (d), the proof is similar to (c). By the construction of partition $\Gamma(\bm{w})$, the walks $\bm{w}^{\Gamma_q}$ must overlap at some edge. However, any overlap must occur on an outgoing edge from root $i$. Otherwise, for any other edge overlap, an undirected cycle can be made to and from the root $i$, contradicting the inherited tree property of $G(\bm{w}^{\Gamma_q})$.

\subsection{Proof of Lemma~\ref{lem:Nstar:Gamma:W:Gamma}}
\begin{proof}
    By Lemma~\ref{lem:tree:character}, each $\bm{w} \in \mathcal{N}_*(\Gamma, \bm{j})$ has root and branching nodes determined, that is, out of $t_*+1$ nodes unique nodes, $(r/2)+1$ nodes are determined. What is left is to select $t_*-r/2 = rk - r$ nodes from $[n]\setminus \{i,\bm{j}^1,\jb^2,\ldots,\jb^{r/2}\}$. An ordered selection of nodes can be used to determine walks $\bm{w}^s$ giving a cardinality of
    \[
    |\mathcal{N}_*(\Gamma, \bm{j})| = |\tuples{[n]\setminus \{i,\bm{j}^1,\jb^2,\ldots,\jb^{r/2}\}}
    {rk-r}
    | =
    \prod_{\ell=1}^{rk-r} \bigl((n - 1 - r/2) -\ell +1 \bigr)
    \]
    For the cardinality of $\mathcal{W}_{k}^r(\Gamma, \jb)$, what is left to determine is each $\mathcal{W}_{k-1}^2(\bm{j}^q)$. Although these walks have no self-loops, they can overlap and repeat edges. So
    \begin{align*}
        |\walk_{k}^r(\Gamma, \jb)| = \prod_{q=1}^{r/2} |\walk_{k}^2(
        \{\Gamma_q\},(\jb^q)
        )| &= \prod_{q=1}^{r/2} |\mathcal{W}_{k-1}(\jb^q)|^2
        = \big((n-1)^{2(k-1)}\big)^{r/2}
    \end{align*}
    As such, for $n \ge rk-r/2$,
    \begin{align*}
    \frac{|\mathcal{N}_*(\Gamma, \bm{j})|}{|\walk_{k}^r(\Gamma, \jb)|} = 
        \prod_{\ell=1}^{rk-r} \Bigl(
    1 - \frac{r/2 + \ell - 1}{n-1}\Bigr) \ge \Bigl( 1 - \frac{rk-r/2-1}{n-1}\Bigr)^{rk-r} 
    \ge 1 -(rk-r) \frac{rk-r/2-1}{n-1} 
    \end{align*}
    using  Bernoulli's inequality. Using $\frac{a-1}{n-1} \le \frac{a}n$ for $n \ge a$, %
    we have
    \[
    (rk-r) \frac{rk-r/2-1}{n-1}  
    \le \frac{(rk-r/2)^2}{n}
    \]
    proving the first assertion. For the second claim, we have
    \begin{align*}
        |\walk_k^{r}(\Gamma, \bm j)\setminus \mathcal{N}_*(\Gamma, \bm j)| &= |\mathcal{W}_k^{r}(\Gamma, \bm j)|-|\mathcal{N}_*(\Gamma, \bm j)|\\
        &= |\mathcal{W}_k^{r}(\Gamma, \bm j)|\Big(1 - \frac{|\mathcal{N}_*(\Gamma, \bm j)|}{|\mathcal{W}_k^{r}(\Gamma, \bm j)|}\Big)\\
        &\leq n^{rk-r} \cdot \frac{t_*^2}{n}
        = t_*^2\, n^{rk-r-1}
    \end{align*}
    which is the desired result since $rk-r = t_*-r/2$.
\end{proof}

\newcommand\Lc{\mathcal L}
\newcommand\Ic{\mathcal I}
\subsection{Proof of Lemma~\ref{lem:M:EA:k-1:lower}}
Let us write $\bar \ind_{\cluster_\ell} = \ind_{\cluster_\ell} / n_\ell$. Let 
\[
\Ymat_{\ell} := \sum_{j:\, y_j=\ell} e_j^{} e_j^T, \quad 
    \Ymat_\Lc := \sum_{\ell \in \Lc} \Ymat_\ell
\]
for any $\Lc \subset [L]$.
Recalling the definition of $\Imat_\ell$ from~\ref{eq:Yell:def}, we have 
$\Imat_\ell = \Ymat_{\Ic_\ell}.$
We note the following identities: 
    For any matrix $H \in \reals^{m \times n}$, and $\Lc \subset [L]$
\begin{align}\label{eq:Z:ell:ident}
    \fnorm{H\Ymat_\Lc}^2 = \sum_{\ell \in \Lc} \fnorm{H\Ymat_{\ell}}^2, \qquad \fnorm{H\Ymat_\ell}^2 = \sum_{j: y_j = \ell} \fnorm{H e_j}^2.
\end{align}
Next, by the symmetry of the SBM:
\begin{lem}\label{lem:fnorm:symmetry}
    For any matrix $H \in \reals^{d \times n}$, and any $k \in \nats$, we have
    \[
     \frac1n \fnorm{H \, \ex[A]^{k} \Ymat_\ell}^2 = \pi_{\ell}\, 
    \norm{ H \,\ex[A]^k \bar \ind_{\cluster_{\ell}} }_2^2.
    \]
    As a consequence, for any $\Lc \subset [L]$,
    \begin{align}\label{eq:H:ident}
    \frac1n \fnorm{H \, \ex[A]^{k} \Ymat_\Lc}^2 = \sum_{\ell  \in \Lc} \pi_{\ell}\, 
    \norm{ H \,\ex[A]^k \bar \ind_{\cluster_{\ell}} }_2^2.
    \end{align}
\end{lem}
\begin{proof}
   Expanding along the columns as above, the left-hand side is equal to
    \[ 
        \frac1n \sum_{j:\,y_j = \ell}\norm{H \, \ex[A]^{k} e_j}_2^2 
        = \frac1n \, n_\ell 
    \norm{ H \,\ex[A]^k \bar \ind_{\cluster_{\ell}} }_2^2
    \]
     where the %
     equality is by $H \, \ex[A]^k e_j =  H \, \ex[A]^k \bar \ind_{\cluster_\ell}$ for any $j\in\mathcal{C}_{\ell}$, a consequence of the symmetry of SBM. The second claim follows by combining the first with identity~\eqref{eq:Z:ell:ident}. The proof is complete.
\end{proof}

We also need the following intermediate lemma:
\begin{lem}\label{lem:indiv_yl}
    Under~\ref{as:scale},~\ref{as:pi} and~\ref{as:mu}, we have
    \[
    \frac{\nu_n}{n}\fnorm{\bmu \ex[A]^{k-1}\Imat_\ell}^2 \;\geq\; \frac{\nu_n^{-1}}{2} \norm{M\ex[A]^{k}\bar{\ind}_{\mathcal{C}_{\ell}}}_2^2 - d C_\mu^2 (L/c_\pi)C_B^2 \nu_n^{2k-1 - 2\delta}.
    \]
\end{lem}
\begin{proof}
 
    We have
    \begin{align*}
         \fnorm{\bmu \, \ex[A]^{k} \,\Ymat_{\ell}} &\;\le\;
        \fnorm{\bmu \,\ex[A]^{k-1}\, \Imat_\ell\, \ex[A]\, \Ymat_{\ell}} + \fnorm{\bmu \,\ex[A]^{k-1}(I-\Imat_{\ell})\,\ex[A]\,\Ymat_{\ell}}\\
        &\;\le\; \norm{\ex[A]\Ymat_{\ell}}\cdot\fnorm{\bmu \, \ex[A]^{k-1}\,\Imat_{\ell}} + \norm{(I-\Imat_{\ell})\ex[A]\Ymat_{\ell}}\cdot \fnorm{\bmu \ex[A]^{k-1}}.
    \end{align*}
    For $R \in \reals^{m \times n}$, we have $\opnorm{R} \le \fnorm{R}\le \sqrt{m n} \maxnorm{R}$. Let
    $\mathcal{D}_\ell \coloneqq \bigcup_{\ell' \in \mathcal{I}_\ell^c} \mathcal{C}_{\ell'}$.
    Then, $(I-\Imat_\ell)\, \ex[A]\, \Ymat_\ell$ is equal to $\ex[A]$ on the submatrix indexed by $\mathcal D_\ell \times \cluster_\ell$, and zero elsewhere. Hence by~\ref{as:scale},
    \[
    \maxnorm{(I-\Imat_\ell)\, \ex[A]\, \Ymat_\ell} \;\le\; C_B \nu_n^{1-\delta} / n,
    \]
    and consequently, $\opnorm{(I-\Imat_\ell)\, \ex[A]\, \Ymat_\ell} \le \sqrt{n\cdot n_\ell} \, C_B \nu_n^{1-\delta} / n = C_B \sqrt{\pi_\ell}  \nu_n^{1-\delta}$. Similarly
    $
    \opnorm{\ex[A]\, \Ymat_\ell} \le \sqrt{n \cdot n_\ell} \, \pmax = \sqrt{\pi_\ell}\,  \nu_n
    $. It follows that
    \[
    \fnorm{\bmu \ex[A]^{k}\Ymat_\ell} \;\leq\; \sqrt{\pi_\ell} \nu_n  \fnorm{\bmu \ex[A]^{k-1}\Imat_\ell} + \sqrt{\pi_\ell} C_B \nu_n^{1-\delta}\fnorm{\bmu \ex[A]^{k-1}}.
    \]
    Squaring both sides, using the inequality $(a+b)^2 \leq 2(a^2+b^2)$, and multiplying by ${\nu_n^{-1}}/{\pi_\ell n}$, we have
    \begin{align}\label{eq:temp:45}
        \frac{\nu_n^{-1}}{\pi_\ell n}\fnorm{\bmu \ex[A]^{k}\Ymat_\ell}^2 \;\leq\; \frac{2\nu_n}{n}\fnorm{\bmu \ex[A]^{k-1}\Imat_\ell}^2 + 2C_B^2 \nu_n^{1-2\delta}\cdot \frac1{n}\fnorm{\bmu \ex[A]^{k-1}}^2
    \end{align}
    By Lemma~\ref{lem:fnorm:symmetry}---specifically, by~\eqref{eq:H:ident} with $\Ymat_\Lc = I$---we have
    \begin{align*}
        \frac{1}{n}\norm{M\ex[A]^{k-1}%
    }_F^2 = \sum_{\ell}\pi_{\ell} \norm{M\ex[A]^{k-1}\bar{\ind}_{\mathcal{C}_{\ell}}}_2^2  
    \;\ge\; d C_\mu^2 (L/c_\pi) \nu_n^{2k-2}.
    \end{align*}
    The inequality follows since for any $\ell \in [L]$,
    \begin{align*}
        \norm{M\ex[A]^{k-1}\mathbbm{1}_{\mathcal{C}_\ell}}_2 &\leq \norm{M}\cdot \norm{\ex[A]}^{k-1}\cdot\norm{\bar{\ind}_{\mathcal{C}_\ell}}\\
        &\leq \sqrt{d} C_\mu (L/c_\pi)^{1/2} \nu_n^{k-1}
    \end{align*}
    using bounds $\norm{M}\leq C_\mu \sqrt{nd}$ and $\norm{\bar{\ind}_{\mathcal{C}_\ell}}_2 = (L/c_\pi)^{1/2} n^{-1/2}$ from Lemma~\ref{lem:M:ic:bound}, and  recalling that by definition
     $\nu_n = n \pmax$, hence $\opnorm{\ex[A]} \le \nu_n$. 
    Combining with~\eqref{eq:temp:45}, and rearranging 
    \begin{align*}
        \frac{\nu_n}{n}\fnorm{\bmu \ex[A]^{k-1}\Imat_\ell}^2 \;\ge\; \frac{\nu_n^{-1}}{2\pi_\ell n}\fnorm{\bmu \ex[A]^{k}\Ymat_\ell}^2 - d C_\mu^2 (L/c_\pi) C_B^2 \nu_n^{2k-1 -2\delta}%
    \end{align*}
    Finally, by Lemma~\ref{lem:fnorm:symmetry}, $\frac{1}{n}\norm{M\ex[A]^{k}\Ymat_\ell}_F^2 = \pi_{\ell} \norm{M\ex[A]^{k}\bar{\ind}_{\mathcal{C}_{\ell}}}_2^2$ establishing the desired result.
\end{proof}

The proof now follows from Lemma~\ref{lem:indiv_yl}
and assumption %
and~\ref{as:sep}.
\begin{proof}[Proof of Lemma~\ref{lem:M:EA:k-1:lower}]
    
    By Lemma~\ref{lem:indiv_yl},
    \begin{align*}
         \frac{\nu_n}{n}
         \sum_\ell \pi_\ell 
            \fnorm{M\,\ex[A]^{k-1}\Imat_\ell}^2
        \geq \Big(\frac{\nu_n^{-1}}{2} \sum_\ell \pi_\ell \norm{M\,\ex[A]^{k}\bar{\ind}_{\mathcal{C}_{\ell}}}_2^2\Big) - d C_\mu^2 (L/c_\pi)C_B^2 \nu_n^{2k-1 - 2\delta}.
    \end{align*}
    Next, we replace $\ex[A]^k$ with $P^k$ on the LHS, 
    since the price is negligible by Lemma~\ref{lem:Ak:Pk:bound}. We obtain
    \begin{align*}
        \norm{\bmu 
        (\ex[A]^k - P^k) \bar \ind_{\cluster_\ell} }_2
        &\;\leq\;
        \norm{\bmu}\cdot \norm{\ex[A]^k - P^k} \cdot \norm{\bar \ind_{\cluster_\ell}}_2\\
        &\;\le\; \sqrt{d}\,  C_\mu (L/c_\pi)^{1/2}\, k \nu_n^k /n.
    \end{align*}
    using $\norm{\ex[A]^k - P^k}\leq k\nu_n^{k}/n$ %
    from Lemma~\ref{lem:Ak:Pk:bound}.  and~\ref{lem:M:ic:bound}. Using $\norm{a}^2 \ge \frac12 \norm{b}^2 - 2\norm{a-b}^2$,
    \begin{align*}
    \norm{\bmu\,\ex[A]^k \bar \ind_{\cluster_\ell} }_2^2 
    \;\ge\; 
    \frac12 \norm{\bmu\,P^k \bar \ind_{\cluster_\ell} }_2^2 
    -2\bigl(\sqrt{d}\, C_\mu (L/c_\pi)^{1/2}\, k\bigr)^2 \nu_n^{2k} n^{-2}. 
    \end{align*}
    Let $C_{2} := 2C_\mu^2 (L/c_\pi)$ and recall the definition of $\xik_\ell = \bmu P^k \bar \ind_{\cluster_\ell}$ from~\eqref{eq:xik:def}. Then,
    \begin{align*}
        \frac{\nu_n}{n}\sum_\ell \pi_\ell \fnorm{\bmu \ex[A]^{k-1}\Imat_\ell}^2 
        \;\ge\; \nu_n^{-1} \Bigl( \frac12
        \sum_\ell \pi_\ell \norm{\xik_\ell}_2^2  -
        d \,C_{2}\, \nu_n^{2k} \big((n/k)\wedge (\nu_n^{\delta}/C_B)\big)^{-2}\Bigr).
    \end{align*}
     W.l.o.g. assume $\pi_1 \geq \pi_2\geq \ldots \geq \pi_L$. Then
    \begin{align*}
        \sum_{\ell}\pi_\ell \norm{\xik_{\ell}}_2^2 
        \;\ge\; \pi_2  \sum_{\ell \in \{1,2\}}\norm{\xik_\ell}_2^2
        \;\geq\; \pi_2 \, \frac12\norm{\xik_1 - \xik_2}_2^2.
    \end{align*}
    Recalling the identity~\eqref{eq:xik:xibk:ident}, namely $\xik_\ell = \nu_n^k \,\xibk_\ell$, and invoking assumptions~\ref{as:pi} and~\ref{as:sep},
    \[
    \sum_{\ell}\pi_\ell \norm{\xik_{\ell}}_2^2  \ge \frac{c_\pi}{2L} %
    c_\xi^2 d\cdot \nu_n^{2k}.
    \]
     Putting the pieces together, 
     \begin{align*}
        \frac{\nu_n}{n} \fnorm{\bmu \,\ex[A]^{k-1}}^2 
        \;\ge\; \nu_n^{2k-1} \Bigl(
        \frac{c_\pi}{4L} %
        c_\xi^2  d
            - %
            C_{2}\, \big((n/k)\wedge (\nu_n^{\delta}/C_B)\big)^{-2} d
        \Bigr) 
        \;\ge\;\nu_n^{2k-1} \frac{c_\pi }{8L} c_\xi^2  d
    \end{align*}
    using %
    $\frac{c_\pi}{8L} c_\xi^2 \ge C_{2} \big((n/k)\wedge (\nu_n^{\delta}/C_B)\big)^{-2}$ which is equivalent to %
    $(n/k)\wedge (\nu_n^\delta/C_B) \ge 4 C_\mu L / (c_\pi c_\xi)$. The proof is complete.
\end{proof}

\section{Auxiliary Lemmas}

\begin{lem}\label{lem:A:moment:inequality}
 For (symmetric) binary $A$, we have, for any $U, V \subset [r]$,
 \[
 \ex[ A_{\wb^V}] \cdot \prod_{u \in U}  \ex[ A_{\wb^{u}}] \;\le\;  \ex[A_{\wb^{U \cup V}}].
 \]
\end{lem}
\begin{proof}
    Since the random variables $A_{ij}$ are binary, for every edge $(i,j)$ and $a,b\in\mathbb{N}$ we have
$\bigl(\mathbb{E}[A_{ij}^a] \bigr)^b \leq \mathbb{E}[A_{ij}]$.
We have $\ex[A_{\wb^u}] = \prod_{e \in [\wb^u]} \ex[A_e]$, using independence of $A_e, e \in [\wb^u]$. Similarly,
\begin{align}\label{eq:temp:48}
    \ex[A_{\wb^U}] = \ex 
    \Bigl[ 
    \prod_{u \in U} A_{\wb^u}
    \Bigr] \stackrel{(a)}{=} 
    \ex \Bigl[ \prod_{e \in [\wb^{U}]} A_e \Bigr] 
    \stackrel{(b)}{=}
     \prod_{e \in [\wb^{U}]} \ex[A_e]
\end{align}
 where (a) is by $A_e^a = A_e$ and (b) is by independence.
This in turn implies that for any  $U \subset [r]$,
\[
\prod_{u \in U} \ex[A_{\wb^u}] = \prod_{u \in U, \, e \in [\wb^u]} \ex[A_e] \le \prod_{e \in [\wb^{U}]} \ex[A_e] 
=\ex[A_{\wb^{U}}]
\]
where the inequality is by $(\ex[A_e])^b \le \ex[A_e]$ and the final equality is by~\eqref{eq:temp:48}.
We obtain
\begin{align*}
   \ex[ A_{\wb^V}] \cdot \prod_{u \in U}  \ex[ A_{\wb^{u}}] &\le \ex[A_{\wb^V}] \cdot \ex[A_{\wb^U}] \\
   &= \Bigl( \prod_{e \in [\wb^V]} \ex[A_e] \Bigr) \cdot \prod_{e' \in [\wb^{U}]} \ex[A_{e'}] \\
   &\le 
   \prod_{e \in [\wb^V] \cup [\wb^U]} \ex[A_e]
\end{align*}
where the final inequality again uses $(\ex[A_e])^b \le \ex[A_e]$.
Using $[\wb^U] \cup [\wb^V] = [\wb^{U \cup V}]$ and~\eqref{eq:temp:48} finishes the proof.
\end{proof}

\begin{lem}\label{lem:mixed:log:convexity}
    For any collection of random variable $\{X_i\}_{i \in I}$ and positive numbers $\{a_i\}_{i \in I}$ such that $\sum_i a_i = a$, we have
    \[
    \prod_i (\ex |X_i|^{a_i}) \le \max_i \ex|X_i|^{a}
    \]
\end{lem}
\begin{proof}
    Let $\lambda_i = a_i / a$ so that $\sum_i \lambda_i = 1$. By Jensen's inequality
    \[
    \mathbb{E}|X_i|^{a_i}\leq (\mathbb{E}|X_i|^{a_i/\lambda_i})^{\lambda_i}.
    \]
    Then, 
    \[
    \prod_i (\ex |X_i|^{\alpha_i}) \le \prod_i (\mathbb{E}|X_i|^{a_i/\lambda_i})^{\lambda_i} \le 
    \max_i \mathbb{E}|X_i|^{a_i/\lambda_i}
    \]
    where the last step uses the elementary inequality  $\prod_{i} |z_i|^{\lambda_i} \le \max |z_i|$.
\end{proof}

\begin{lem}\label{lem:growth:of:rn}
     Let $r_n = \max\{r\in 2\mathbb{N}:\, c( ar)^{r}\leq \nu_n^b \}$ for $a,c\geq 1$ and $b > 0$. Then, 
    \[
    r_n \gtrsim  \frac{b\log \nu_n}{\log(abc) + \log\log \nu_n}.
    \]
\end{lem}
\begin{proof}
    To lowerbound $r_n$, evaluate the inequality $(acr_n)^{r_n}\leq \nu_n^b$. Redefine $a\coloneqq ac$. Let $r \in \reals$ such that $(ar)^{ar} = \nu_n^{ab}$. Then, $r_n \asymp r$.  Recall that $ye^y = x$ has solution $y = W_0(x)$, the Lambert function. 
    Let $y = \log(ar)$, so that $y e^y = \log((ar)^{ar}) = ab\log (\nu_n)$. Then, 
    \[a r = e^{W_0(ab\log \nu_n)} = \frac{ab\log \nu_n}{W_0(ab\log \nu_n)} \ge \frac{ab\log \nu_n}{\log(ab\log \nu_n)}\]
    since $W_0(x) < \log(x)$. Simplifying and relating $r$ to $r_n$ completes the proof
\end{proof}

\end{document}